\documentclass{article}

\usepackage[utf8]{inputenc} %
\usepackage[T1]{fontenc}    %
\usepackage[american]{babel}

\usepackage[nonatbib,final]{neurips_2020}  %

\usepackage{url}            %
\usepackage{csquotes}
\usepackage{booktabs}       %
\usepackage{amsmath,amssymb,amsthm}
\usepackage{enumitem}
\usepackage{amsfonts}       %
\usepackage{nicefrac}       %
\usepackage{microtype}      %

\usepackage{mathtools}
\usepackage{thmtools}
\usepackage{thm-restate}

\setitemize{nosep} %

\usepackage[hidelinks]{hyperref}       %

\usepackage[noabbrev,capitalize,nameinlink]{cleveref}
\crefname{appsec}{Appendix}{Appendices}

\makeatletter
\@ifpackageloaded{jmlrish}{\newcommand\mybibstyle{alphabetic}}{\newcommand{\mybibstyle}{numeric}}
\makeatletter
\usepackage[backend=biber,style=\mybibstyle,sortcites,natbib,maxbibnames=12,maxcitenames=2,maxalphanames=10]{biblatex}
\renewbibmacro{in:}{}
\DeclareMathOperator*{\argmin}{arg\,min}
\newcommand{\A}{\mathcal{A}}
\newcommand{\D}{\mathcal{D}}
\DeclareMathOperator{\E}{\mathbb{E}}
\newcommand{\N}{\mathcal{N}}
\newcommand{\W}{\mathcal{W}}
\newcommand{\Z}{\mathbb{Z}}
\newcommand{\R}{\mathbb{R}}
\newcommand{\st}{\,\mathrm{s.t.}\,}
\newcommand{\tp}{^\mathsf{T}}
\DeclareMathOperator{\Tr}{Tr}
\let\tr\Tr
\newcommand{\rad}{\mathfrak{R}}
\newcommand{\samp}{\mathbf{S}}
\newcommand{\sampb}{\mathbf{\tilde{S}}}
\newcommand{\sampsdelta}{\mathcal{S}_{n,\delta}}
\newcommand{\wmr}{\hat{w}_{\mathit{MR}}}
\newcommand{\wmn}{\hat{w}_{\mathit{MN}}}
\newcommand{\wmnS}{\hat{w}_{\mathit{MN},S}}
\newcommand{\wmnJ}{\hat{w}_{\mathit{MN},J}}

\newcommand{\asge}{\stackrel{a.s.}{\ge}}
\newcommand{\aseq}{\stackrel{a.s.}{=}}

\newcommand{\asto}{\stackrel{a.s.}{\to}}

\makeatletter
\newcommand\given{\@ifstar{\mathrel{}\middle|\mathrel{}}{\mid}}

\DeclareRobustCommand{\abs}{\@ifstar\@abs\@@abs}
\newcommand{\@abs}[1]{\left\lvert #1 \right\rvert}
\newcommand{\@@abs}[1]{\lvert #1 \rvert}

\DeclareRobustCommand{\norm}{\@ifstar\@norm\@@norm}
\newcommand{\@norm}[1]{\left\lVert #1 \right\rVert}
\newcommand{\@@norm}[1]{\lVert #1 \rVert}

\DeclareRobustCommand{\inner}{\@ifstar\@inner\@@inner}
\newcommand{\@inner}[2]{\left\langle #1, #2 \right\rangle}
\newcommand{\@@inner}[2]{\langle #1, #2 \rangle}
\makeatother

\newcommand{\tagthis}{\stepcounter{equation}\tag{\theequation}}

\newcommand{\httpsurl}[1]{\href{https://#1}{\nolinkurl{#1}}}

\declaretheorem[name=Theorem,refname={Theorem,Theorems},numberwithin=section]{theorem}
\declaretheorem[name=Lemma,refname={Lemma,Lemmas},sibling=theorem]{lemma}
\declaretheorem[name=Proposition,refname={Proposition,Propositions},sibling=theorem]{prop}
\declaretheorem[name=Definition,refname={Definition,Definitions},sibling=theorem]{defn}

\newlist{settinglist}{enumerate}{1}
\setlist[settinglist]{label=\textbf{\Alph*}}
\Crefname{settinglisti}{Setting}{Settings}

\makeatletter
\@ifpackageloaded{jmlrish}{
  \title{On Uniform Convergence and Low-Norm Interpolation Learning}
  \author{%
    Lijia Zhou \addr{Department of Statistics, University of Chicago} \email{zlj@uchicago.edu}
    \AND
    D.J.\ Sutherland \addr{Toyota Technological Institute at Chicago} \email{djs@djsutherland.ml}
    \AND
    Nathan Srebro \addr{Toyota Technological Institute at Chicago} \email{nati@ttic.edu}
  }
}{
  \title{On Uniform Convergence\\and Low-Norm Interpolation Learning}
  \author{%
    Lijia Zhou\\University of Chicago\\\texttt{zlj@uchicago.edu}
    \And Danica J.\ Sutherland\\TTI-Chicago\\\texttt{danica@ttic.edu}
    \And Nathan Srebro\\TTI-Chicago\\\texttt{nati@ttic.edu}
  }
}
\makeatother

\begin{document}

\maketitle

\begin{abstract}
We consider an underdetermined noisy linear regression model where the minimum-norm interpolating predictor is known to be consistent, and ask: can uniform convergence in a norm ball, or at least (following Nagarajan and Kolter) the subset of a norm ball that the algorithm selects on a typical input set, explain this success? We show that uniformly bounding the difference between empirical and population errors cannot show any learning in the norm ball, and cannot show consistency for any set, even one depending on the exact algorithm and distribution. But we argue we can explain the consistency of the minimal-norm interpolator with a slightly weaker, yet standard, notion: uniform convergence \emph{of zero-error predictors} in a norm ball. We use this to bound the generalization error of low- (but not minimal-) norm interpolating predictors.
\end{abstract}

\section{Introduction}

In the past several years, it has become empirically clear that -- contrary to traditional intuition -- it is possible for models which exactly interpolate noisy training data to reliably generalize well on practical problems, especially in deep learning \citep{ NTS:real-inductive-bias,ZBHRV:rethinking,BMM:understand-kernels}.
We refer to this phenomenon as ``interpolation learning.''
It is closely related to the (re-)discovery of the ``double descent'' phenomenon \citep{reconcile:interpolation,NKBYBS:deep-double-descent,Spi18,AS17},
where many models first improve as their size is increased, then get much worse around the point where they can first interpolate the data, and then improve again as they become more and more overparametrized.
Understanding interpolation learning, therefore, seems to be a key step on the path towards better theoretical understanding of the successes of deep learning.

We now know of a few settings where interpolating models can be shown to generalize well \citep{BHM:perfect,BRT:opt-interp}.
In particular,
significant recent attention has been paid to
the minimum-norm linear interpolator (``ridgeless'' regression) in certain high-dimensional linear regression regimes
\citep{bartlett:overfitting,hastie:surprises,muthukumar:interpolation,BHX:two-models}.
This setting is of particular interest
not only because it is reasonably accessible to study while exhibiting many of the surprising properties of more complex models,
but also because this predictor is the same one found by (stochastic) gradient descent initialized at the origin,
and so it seems plausible that its properties may generalize to more complex settings.
Much is now understood about the properties of the minimum-norm interpolator for (sub-)Gaussian data,
including necessary and sufficient conditions for its consistency.
This line of inquiry has proved quite fertile for extensions to related settings and further results \citep{BESWZ:two-layer,MM:random-feats,MRSY:max-margin,JLL:basis-pursuit,mahoney:surrogate}.

One striking feature of this body of work is that none of it is based on the core workhorse of learning theory, \emph{uniform convergence};
most instead uses various tools, mostly from random matrix theory, to directly analyze the generalization error of a particular predictor.
Indeed, some have argued that uniform convergence is unlikely to be able to explain interpolation learning;
for instance, Mikhail Belkin has said\footnote{Talk at the Simons Institute for the Theory of Computing, July 2019: {\scriptsize \httpsurl{simons.berkeley.edu/talks/tbd-65}}} that
``there are no [uniform generalization] bounds'' with constants tight enough to explain interpolation learning, ``and no reason they should exist.''
Meanwhile, \citet{nagarajan:uniform} have also raised significant questions about the ability of uniform convergence arguments to explain learning in certain high-dimensional regimes.
Perhaps, then, it is time to wholly abandon uniform convergence in favor of other tools.

We connect these two avenues of work by studying uniform convergence in a particular overparametrized linear regression problem (\cref{sec:problem-setting}) where the minimal-norm interpolator is consistent.
We prove that, indeed, uniform convergence bounds based on predictor norm cannot show \emph{any} learning in this setting (\cref{thm:ball-gap-diverges}).
We also prove,
following \citeauthor{nagarajan:uniform},
that \emph{no} uniform convergence bound can show consistency (\cref{thm:zico-style}),
not only for the minimal-norm interpolator but even for a wide variety of natural interpolation algorithms.

Yet, even in this setting where the situation looks bleak,
we need not abandon uniform convergence entirely.
One option would be sidestep the negative results by considering uniform convergence not of our predictor, but of a surrogate separately shown to be not too different \citep{negrea:in-defense}.
We instead demonstrate that it is possible to show uniform convergence of our predictor directly
if we allow ourselves a slightly weaker notion of uniform convergence,
one long in common use in realizable PAC analyses:
uniform convergence \emph{for predictors with zero error}.
Such a bound would be implied by, for example, ``optimistic rates'' \citep{optimistic-rates},
although existing results are not tight enough to show consistency in our setting.
Instead we prove (\cref{thm:junk-ball}) that a tight version of this notion of uniform convergence \emph{does} hold in our setting for low-norm predictors.
Our result exactly characterizes the asymptotic worst-case generalization gap for predictors of a given norm
via a novel analysis based on strong duality of a particular non-convex problem,
and show that while neither having a low norm nor interpolation is sufficient for generalization in our setting,
the combination is.
By doing so,
not only do we prove consistency of the minimal-norm interpolator with a uniform convergence-type argument,
we also provide new insight about the behavior of interpolation learning for solutions with low but not minimal norm.

\section{Problem setting} \label{sec:problem-setting}

We begin with a standard linear regression setup, with Gaussian data and errors.
Take i.i.d. observations $(x_1,y_1), ..., (x_n, y_n) \sim \D^n$, where the joint distribution $\D$ is given by
\begin{settinglist}
    \item \label{setting:lin-gauss}
      $x \in \R^p$ is drawn from $\N(0, \Sigma)$, with $\Sigma \succ 0$, and $\epsilon \in \R$ is independently $\N(0, \sigma^2)$.
      There is some fixed $w^* \in \R^p$ such that
      $y = \inner{w^*}{x} + \epsilon$.
\end{settinglist}

We consider a ``junk features'' setting, where $x$ decomposes into ``signal'' and ``junk'' components, and analysis of interpolation learning is particularly appealing:
\begin{settinglist}[resume]
    \item \label{setting:junk-feats}
      In \cref{setting:lin-gauss}, let
      $\Sigma = \begin{bmatrix}
I_{d_S} & 0_{d_S \times d_J} \\
0_{d_J \times d_S} & \frac{\lambda_n}{d_J} I_{d_J}
\end{bmatrix}$ where $d_S, d_J$ satisfy $d_S + d_J = p$, and $\lambda_n > 0$. \\
In other words, we can write $x = (x_S, x_J)$, where $x_S \sim \N(0, I_{d_S})$ and $x_J \sim \mathcal{N}(0, \frac{\lambda_n}{d_J} I_{d_J})$.
    Further, the label depends only on $x_S$:
    $w^* = (w_S^*, 0_{d_J})$
    with $w_S^* \in \R^{d_S}$.
\end{settinglist}

Let $Y \in \R^n$ be the vector of responses, $X \in \R^{n \times p}$ the design matrix and $E \in \R^n$ the residual vector, so $Y = Xw^* + E$.
The sample covariance is $\hat\Sigma = \frac1n X\tp X$.
The population and empirical risks are, respectively,
\begin{equation}
    \label{eq:ld-and-ls}
    \begin{split}
    L_\D(w)
    &= \E_{(x, y) \sim \D}[ (y - \inner{w}{x})^2 ]
    = L_\D(w^*) + \norm{w - w^*}_{\Sigma}^2
    \\
    L_\samp(w)
    &= \frac1n \norm{Y - X w}^2
    = L_\samp(w^*) + \norm{w - w^*}_{\hat\Sigma}^2 - \frac{2}{n} \inner{X\tp E}{w - w^*}
    ,\end{split}
\end{equation}
where $\norm{x}_A = \sqrt{x\tp A x}$ denotes the Mahalonobis norm,
and $L_\D(w^*) = \E \epsilon^2 = \sigma^2$.

We will focus on the regime where $d_S$ is fixed, and $d_J \to \infty$ for \emph{finite} values of $n$,
e.g.\ $\lim_{n\to\infty} \lim_{d_J \to \infty} L_\D(\hat w)$.
This setting enables relatively easy calculation of many quantities of interest,
and can recover many interesting behaviors of overparametrized interpolation, 
including consistency and the double descent phenomenon.

We will be primarily concerned with the behavior of the minimal-norm interpolator,
\begin{equation} \label{eq:wmn}
    \wmn
    = \argmin_{\substack{w \in \R^p \st X w = Y}} \, \norm{w}_2^2
    = X\tp (XX\tp)^{-1} Y
.\end{equation}
This predictor is in fact consistent
in \cref{setting:junk-feats} when $\lambda_n = o(n)$
and we consider $d_J \to \infty$ for each $n$.
We here use a slightly broader notion of \emph{consistency}
than is traditional \citep[e.g. in][]{understanding-ML}:
we mean that
\[
    \E\left[ L_\D(\wmn) - L_\D(w^*) \right] \to 0
\]
for our \emph{sequence} of learning problems in the given asymptotic regime.
Specifically:
\begin{prop} \label{prop:wmn-consistent}
  In \cref{setting:junk-feats} with $\lambda_n = o(n)$,
  \[
    \lim_{n \to \infty} \lim_{d_J \to \infty} \E\left[ L_\D(\wmn) - L_\D(w^*) \right] = 0
  .\]
\end{prop}

The proof follows from \cref{prop:equiv-ridge,prop:ridge-consistent},
which establish first --
because the setting was designed exactly to make this true\footnote{%
If the noise scaling were $\omega(1 / d_J)$, then as $d_J \to \infty$, the minimal-norm solution would exploit the exploding magnitude of the noise components, and all of the signal would ``bleed'' into the noise dimensions \citep{hastie:surprises,muthukumar:interpolation},
giving $\norm\wmn \to 0$ and $L_\D(\wmn) \to L_\D(0_p)$ -- in the ridge regression equivalence, we let the regularization weight go to infinity.
On the other hand, if the noise scaling were $o(1 / d_J)$,
then we would have $\norm\wmn \to \infty$,
significantly complicating matters.
$\Theta(1 / d_J)$ is the only scaling in which $\norm\wmn$ is bounded but nonzero.
}
-- that $\wmn$ becomes equivalent to ridge regression on the signal part of $X$ with regularization weight $\lambda_n$,
and then that ridge regression is consistent in this setting.

Writing $X = (X_S, X_J)$ with $X_S \in \R^{n \times d_S}$ and $X_J \in \R^{n \times d_J}$, the ridge regression estimate on the signal components with tuning parameter $\lambda$ is given by 
\begin{align*}
    \hat{w}_{\lambda}
    &= \argmin_{w \in \R^p} \, \norm{Y-X_S w}^2 + \lambda \norm{w}^2
\\  &= (X_S\tp X_S + \lambda I_{d_S})^{-1}X_S\tp Y
     = X_S\tp (X_S X_S\tp + \lambda I_n)^{-1} Y
.\end{align*}

\begin{lemma} \label{prop:equiv-ridge}
  In \cref{setting:junk-feats},
  $\lim_{d_J \to \infty} \E[ L_\D(\wmn) ] = \E[ L_\D(\hat w_{\lambda_n}) ]$
  for any $n$.
\end{lemma}
\begin{proof}
By the strong law of large numbers, we have
that $X_J X_J\tp = \lambda_n \frac{Z_J Z_J\tp}{d_J}$ converges almost surely to $\lambda_n I_n$.
Writing $\wmn = (\wmnS, \wmnJ)$, we can easily verify that 
\begin{itemize}
    \item
        $\wmnS = X_S\tp (X_S X_S\tp + X_J X_J\tp)^{-1} Y \asto \hat w_{\lambda_n}$
        by the continuous mapping theorem.
    \item
        $\wmnJ = X_J\tp (X_S X_S\tp + X_J X_J\tp)^{-1} Y$.
        Drawing a new $x_J \sim \N(0, \frac{\lambda_n}{d_J} I_{d_J})$,
        $X_J x_J \asto 0_n$
        and so
        $
          \inner{\wmnJ}{x_J} \asto 0 %
        $.
\end{itemize}
This implies that for any fixed $x$, $\inner{\wmn}{x} \asto \inner{\hat{w}_{\lambda_n}}{x_S}$,
and hence via continuity we have that $(\inner{\wmn}{x} - y)^2 \asto (\inner{\hat{w}_{\lambda_n}}{x} - y)^2$.
Taking expectations over $(x, y)$ to get $L_\D$
and then over the training set,
then
exchanging the limit with each expectation,\footnote{%
    Both exchanges can be justified using dominated convergence thoerem and the techniques from the proof of \cref{thm:junk-wmr}, which independently shows a stronger statement.
}
we obtain the desired result.
\end{proof}

\begin{restatable}{lemma}{ridgeconsistent} \label{prop:ridge-consistent}
In \cref{setting:junk-feats}, if $\lambda_n = o(n)$, then
$
  \lim_{n \to \infty} \E \left[ L_\D(\hat{w}_{\lambda_n}) - L_\D(w^*) \right] = 0
$.
\end{restatable}
The proof, as for all the following results, is in the appendix.
Taking $\lambda_n = o(n)$ ensures the bias due to regularization is negligible;
the minimax-optimal scaling would be $\lambda_n \propto \sqrt n$ \citep{caponnetto-devito}.

\paragraph{Relationship to previous settings}
The results of \citet{bartlett:overfitting} apply to our setting,
also showing consistency of $\wmn$.
Although they do not require $p \to \infty$ for finite $n$ as we study,
their results show that consistency of $\wmn$ is only possible when the effective $p$ grows much faster than $n$.
\Citet{muthukumar:interpolation} showed that \emph{no} interpolation method can be consistent in \cref{setting:lin-gauss} for $p = \mathcal O(n)$;
we re-derive this (simple) result in \cref{prop:wmr-consistent},
since it will also be important for our purposes.

\Citet{hastie:surprises} and various follow-ups, on the other hand,
employ the standard asymptotic regime of random matrix theory,
where $n / p \to \gamma \in (0, \infty)$,
mostly focusing on $\Sigma = I$.
Although no interpolator can achieve consistency here,
they exactly evaluate $\lim_{(n,d) \to \infty} L_\D(\wmn)$.
The setting of \citet{BHX:two-models} is related,
with general $(n, p)$ but again with $\Sigma = I$,
where $\wmn$ is not consistent.

\section{Uniform convergence} \label{sec:uniform-gen-gap}

We now know, via \cref{prop:wmn-consistent}, that $\wmn$ is consistent in this setting.
Could we have discovered this fact directly via uniform convergence?
Typically, we would find some class $\W_{n,\delta}$
such that $\Pr(\wmn \in \W_{n,\delta}) \ge 1 - \delta$,
and bound the generalization gap
\begin{equation} \label{eq:wdelta-gen-gap}
    \Pr\left(
        \sup_{w \in \W_{n,\delta}} L_\D(w) - L_\samp(w) \le \epsilon_\W(n, \delta)
    \right)
    \ge 1 - \delta
.\end{equation}
As $L_\samp(\wmn) = 0$, this would directly provide an upper bound on $L_\D(\wmn)$ with probability $1 - 2 \delta$.

\subsection{Uniform convergence over norm balls} \label{sec:one-sided}
Our first thought would likely be to find some high-probability upper bound $B_{n,\delta}$ on $\norm\wmn$,
and take $\W_{n,\delta} = \{ w \in \R^p : \norm w \le B_{n,\delta} \}$.
We can get a rough asymptotic estimate for $B_{n,\delta}$ based on the following,
since
$\norm{\wmn} = \mathcal O_P\left( \sqrt{\E \norm{\wmn}^2} \right)$ by Markov's inequality.
\begin{restatable}{prop}{wmnsizebasic} \label{prop:wmn-size}
    As $n \to \infty$ in \cref{setting:junk-feats},
    if $\lambda_n$ is both $o(n)$ and $\omega(1)$, then
    \[
        \lim_{d_J \to \infty} \E \norm{\wmn}^2
        = \frac{\sigma^2 n}{\lambda_n} + \mathcal{O}(1)
        \quad\text{and}\quad
        \lim_{d_J \to \infty} \frac{(\E \norm\wmn^2) (\E \norm{x}^2)}{n}
        = \sigma^2 + o(1)
    .\]
\end{restatable}
We could then find $\epsilon_\W(n, \delta)$ by studying the Rademacher complexity,
given by
\[
    \rad_n(\W_B)
    = \E_{\samp} \E_{\sigma \sim \operatorname{Unif}(\pm 1)^n}
      \sup_{w : \norm w \le B}
      \frac1n \sum_{i=1}^n \sigma_i \inner{w}{x^{(i)}}
    \le \sqrt{\frac1n \, B^2 \, \E \norm{x}^2}
;\]
thus \cref{prop:wmn-size} gives us that
$\rad_n(\W_{\sqrt{\E \norm{\wmn}^2}}) \le \sigma + o(1)$.

Standard Rademacher bounds are for Lipschitz losses,
which the squared loss is not.
If we let $T_n$ be a uniform upper bound on all the labels
and $Q_n$ on all the predictions, however,
the absolute value of the derivative of the squared loss is at most
$2 \abs{\hat y - y} \le 2 (Q_n + T_n)$,
and so we can treat it as $2(Q_n + T_n)$-Lipschitz with high probability.
We then obtain in the setting of \cref{prop:wmn-size} that
\begin{equation} \label{eq:rad-up-bound}
    \sup_{\norm{w}^2 \le \E \norm{\wmn}^2} L_\D(w) - L_\samp(w)
    \le 4 (Q_n + T_n) \left( \sigma + \mathcal{O}_P\left( \frac{1}{\sqrt n} \right) \right)
.\end{equation}
To show consistency, we need a bound exactly approaching $\sigma^2$ as $n \to \infty$,
i.e.\ $Q_n + T_n \to \frac14 \sigma$.
But in fact,
each of $Q_n$ and $T_n$ diverge to $\infty$ as $n \to \infty$,
because we have more and more chances to see a large value.
Thus for $n \to \infty$, \eqref{eq:rad-up-bound} says nothing at all.

Now, the path to \eqref{eq:rad-up-bound} was potentially quite loose,
particularly in the Lipschitz step;
perhaps, then, we could simply put more effort in to obtain the bound we want.
This is not the case:
balls which are big enough to contain $\wmn$
also contain predictors with unbounded generalization gaps as $n \to \infty$.

\begin{restatable}{theorem}{ballgapdiverges} \label{thm:ball-gap-diverges}
In \cref{setting:junk-feats}, if $\lambda_n = o(n)$ then
\[
     \lim_{n \to \infty} \lim_{d_J \to \infty} \E \left[ \sup_{\norm w \le \norm\wmn} |L_\D(w) - L_\samp(w)| \right] = \infty
.\]
\end{restatable}

\begin{proof}[Proof sketch]
\Cref{thm:ballgapbound} shows that the gap is at least $\norm{\Sigma - \hat\Sigma} (\norm\wmn - \norm{w^*})^2 + o(1)$
using \eqref{eq:ld-and-ls}
and then aligning $w - w^*$ with $\Sigma - \hat\Sigma$.
By \cref{prop:wmn-size},
$(\norm\wmn - \norm{w^*})^2$ grows like $n / \lambda_n$.
Now $\norm{\Sigma - \hat\Sigma}$ goes to $0$,
but only at the rate of $\sqrt{\lambda_n / n}$ \citep{sample-covariance},
so the product grows as $\sqrt{n / \lambda_n}$.
\end{proof}

\Cref{thm:ballgapbound} also gives a lower bound for
$\E \left[ \sup_{\norm w \le \norm\wmn} L_\D(w) - L_\samp(w) \right]$,
the one-sided generalization gap,
based on the algebraically largest eigenvalue of $\Sigma - \hat\Sigma$
rather than the operator norm.
We expect that this eigenvalue should asymptotically behave similarly to the operator norm,
and hence the one-sided generalization gap should also diverge.

Norm balls around $w^*$, rather than the origin, fare no better;
they would merely remove the asymptotically irrelevant $\norm{w^*}$ term from the result of \cref{thm:ballgapbound}.

\subsection{Uniform convergence over algorithm- and distribution-dependent hypothesis classes} \label{sec:two-sided}
Choosing $\W_{n,\delta}$ as a Euclidean norm ball, then, cannot yield the result we want (or, indeed, any meaningful result at all for large $n$).
But a norm ball doesn't fully capture everything we know about $\wmn$:
for instance, we know that its norm is not likely to be very small.
Perhaps taking a shell rather than a ball would help?
Following \citet{nagarajan:uniform}, we show that in fact,
\emph{no} choice of $\W_{n,\delta}$ can demonstrate consistency
using the most common two-sided uniform convergence bounds.

Specifically,
let $\sampsdelta$ be a set of typical training examples $\samp = (X, Y)$
such that $\Pr(\samp \in \sampsdelta) \ge 1 - \delta$,
let $\A(X, Y)$ be any learning algorithm,
and then take the class of typical outputs of $\A$,
$\W_{n,\delta}^\A = \{ \A(X, Y) : (X, Y) \in \sampsdelta \}$.
(Clearly, no bound based on $\sampsdelta$ could choose a smaller $\W_{n,\delta}$.)
The \emph{tightest algorithm-dependent uniform convergence bound} \citep{nagarajan:uniform} is then
\begin{gather}
    \sup_{\samp \in \sampsdelta} \sup_{w \in \W_{n,\delta}^\A} \abs{L_\D(w) - L_\samp(w)} \leq \epsilon_\A^\D(n, \delta)
     \label{eq:alg-dep-bound}
    ,\\ \text{implying }
    \Pr\Big( \abs*{ L_\D(\A(X, y)) - L_\samp(\A(X, y)) } \leq \epsilon_\A^\D(n, \delta) \Big)
    \geq 1 - \delta
    \notag
.\end{gather}
In interpolation learning, where $L_\samp$ is zero,
we need $\lim_{n \to \infty} \epsilon_\A^\D(n, \delta) = \sigma^2$ to obtain consistency.
\citeauthor{nagarajan:uniform} show that in a particular high-dimensional linear classification setting,
stochastic gradient descent has $0$ asymptotic loss,
but $\epsilon_\A^\D(n, \delta)$ must be nearly $1$ for any $\sampsdelta$.
We show a similar result in our setting,
not only for $\A = \wmn$ but indeed for many interpolation methods.\footnote{%
Lemma 5.2 of \citet{negrea:in-defense} is closely related; it covers \cref{setting:lin-gauss} in general, but applies only to $\wmn$ and shows a smaller gap. %
}

\begin{restatable}{theorem}{zicostyle} \label{thm:zico-style}
  In \cref{setting:junk-feats},
  let $\A$ be an algorithm outputting interpolators,
  $X \A(X, Y) = Y$,
  with
  \begin{equation} \label{eq:a-flip-cond}
    \A\left( (X_S, X_J), y \right)_S = \A\left( (X_S, -X_J), y \right)_S
    \quad\text{and}\quad
    \lim_{n \to \infty} \lim_{d_J \to \infty} L_\D(\A(X, y)) \aseq \sigma^2
  .\end{equation}
  For any $\delta \in (0, \frac12)$
  and set of typical training examples $\sampsdelta$ satisfying
  $\Pr(\samp \in \sampsdelta) \ge 1 - \delta$,
  let $\W_{n,\delta}^\A = \left\{ \A(X, Y) : (X, Y) \in \sampsdelta \right\}$
  denote the set of typical outputs.
  Then
  \begin{equation} \label{eq:twosided-conv}
    \lim_{n \to \infty} \lim_{d_J \to \infty} \sup_{\samp \in \sampsdelta} \sup_{w \in \W_{n,\delta}} \abs{L_\D(w) - L_\samp(w)}
    \asge 3 \sigma^2
  .\end{equation}
\end{restatable}

\begin{proof}[Proof sketch]
  For each $\samp = (X, Y) \in \sampsdelta$,
  let $\sampb = \left( \left(X_S, -X_J\right), Y \right)$,
  which has equal density under $\D$, so that $\sampsdelta$ must contain some $(\samp, \sampb)$ pairs.
  Consider $\tilde w = \A( \sampb )$:
  we know that $L_\D(\tilde w) \asto \sigma^2$ by assumption,
  but we will show $\lim_{n\to\infty} \lim_{d_J \to \infty} L_\samp(\tilde w) \asge 4 \sigma^2$.

  This is easiest to see in the case when $d_S = 0$,
  so that $y \sim \N(0, \sigma^2)$ is independent of $x$.
  Then $- X \tilde w = Y$,
  so that $X \tilde w = - Y$,
  and thus $L_\samp(\tilde w) = \frac1n \norm{(-Y) - Y}^2 = \frac4n \norm{Y}^2 \aseq 4 \sigma^2$.

  The general case, in \cref{sec:proofs:two-sided}, shows that since $X_S$ is rank $d_S \ll n$,
  $\tilde w_J$ must be large enough to contribute $4 \sigma^2 \frac{n - d_S}{n} \to 4 \sigma^2$ to the loss.
\end{proof}

From \eqref{eq:wmn}, we can see that $\wmn$ satisfies the symmetry condition in \eqref{eq:a-flip-cond}.
In fact, \cref{thm:flip-alg} (in \cref{sec:proofs:two-sided}) shows this is also true of many more algorithms,
including interpolators which minimize $\norm{w}_1$ (basis pursuit)
or even $\norm{w - w^*}$:
any algorithm that picks the interpolator minimizing $f_S(w_S) + f_J(w_J)$,
where each function is convex and $f_J(-w) = f_J(w)$.

The attentive reader may have noticed that \cref{thm:zico-style}, like \cref{thm:ball-gap-diverges}, applies only to bounds on $\abs{L_\D(w) - L_\samp(w)}$,
whereas the general argument as in \eqref{eq:wdelta-gen-gap} only needs to bound $L_\D(w) - L_\samp(w)$.
Indeed, the proof of \cref{thm:zico-style} exhibits a hypothesis with low generalization error but high \emph{training} error -- not a particularly concerning failure mode.
Whenever $\A$ is consistent, it is trivially guaranteed
that there is a $\W_{n,\delta}$ where \eqref{eq:wdelta-gen-gap} holds with
$\epsilon_{\W}(n, \delta) \to L_\D(w^*)$, and so \citeauthor{nagarajan:uniform}'s approach is not meaningful for one-sided bounds.\footnote{%
    Take $\sampsdelta = \{ (X, Y) : L_\D(X, Y) \le L_\D(w^*) + \epsilon_{n,\delta} \}$;
    consistency implies that there is a choice of $\epsilon_{n,\delta} \to 0$
    such that
    $\Pr(\samp \in \sampsdelta) \geq 1 - \delta$ and
    \[ 
        \epsilon_{n,\delta}
        \ge \sup_{\samp \in \sampsdelta} \sup_{w \in \W_{n,\delta}^\A} L_\D(w)
        \ge \sup_{\samp \in \sampsdelta} \sup_{w \in \W_{n,\delta}^\A} L_\D(w) - L_\samp(w)
    .\]
}
Thus it is not possible to mathematically rule out that one could prove a one-sided bound on
$\sup_{w \in \W} L_\D(w) - L_\samp(w)$ using a uniform convergence-type technique. 
(Again, since one-sided uniform convergence is always a consequence of consistency, this question is essentially one of viewpoint: do you first show uniform convergence and then bound consistency through uniform convergence, or do you establish uniform convergence as a consequence of consistency?)
In any case, as argued by \citeauthor{nagarajan:uniform},
existing uniform convergence proofs essentially bound $\abs{L_\D(w) - L_\samp(w)}$, not $L_\D(w) - L_\samp(w)$.

\section{Uniform convergence for interpolating predictors} \label{sec:interp-convergence}

In \cref{setting:junk-feats},
we now know it is impossible to prove consistency of $\wmn$ with a bound on
$\sup_{w \in \W} \abs{L_\D(w) - L_\samp(w)}$ for any fixed choice of $\W$,
and it seems quite unlikely that we can do so with
bounds on $\sup_{w \in \W} L_\D(w) - L_\samp(w)$ either.
However, since we are concerned only with zero-training-error predictors, perhaps we should instead look at bounds on
\begin{equation} \label{eq:gen-gap}
    \sup_{\norm w \le B,\, L_\samp(w) = 0} L_\D(w) - L_\samp(w)
.\end{equation}
Although $L_\samp(w)$ is identically $0$ in \eqref{eq:gen-gap},
we write it to emphasize that this is still fundamentally a bound on the generalization gap
as in \eqref{eq:wdelta-gen-gap}.
When $L_\samp(w) = 0$, of course, one-sided and two-sided convergence become the same.
Moreover, when $B = \norm\wmn$,
\eqref{eq:gen-gap} becomes identically $L_\D(\wmn)$,
which we know from \cref{prop:wmn-consistent} is small.
Our questions are (a) whether we could have shown this via uniform convergence,
and (b) precisely how small $B$ has to be compared to $\norm{\wmn}$ in order to maintain consistency.

The uniform convergence of \eqref{eq:gen-gap} is a weaker notion than that of \cref{sec:uniform-gen-gap},
as the hypothesis set is sample-dependent.
But it is still a standard and common form of ``uniform convegnce'' at the basis of classical learning theory, and is well understood to be necessary for obtaining tight learning guarantees when we expect the training error to be zero.
For example, this is the notion used by \citet{Valiant84} to first establish standard (realizable) PAC-learning guarantees,
and is the starting point for standard textbooks,
as in Section 2.3.1 of \citet{understanding-ML},
or Theorem 2.1 of \citet{MRT:foundations} where that book first introduces the term ``uniform convergence bound.''

A bound on \eqref{eq:gen-gap}
would be implied by bounds with ``optimistic rates'' \citep{panchenkooptimistic,optimistic-rates},
which interpolate between a ``fast'' rate for $L_\D(w) - L_\samp(w)$ and a ``slow'' one depending on $L_\samp(w)$.
For instance, the result of \cite{optimistic-rates}
implies that if $\xi_n$ is a high-probability upper bound on $\max_{1 \le i \le n} \norm{x_i}^2$,
we have uniformly over all $w$ with $\norm w \le B$ that
\begin{equation} \label{eq:opt-rate-linear}
    L_\D(w) - L_\samp(w)
    \le
    \tilde{\mathcal O}_P\left(
        \frac1n B^2 \xi_n
        + \sqrt{L_\samp(w) \frac{B^2 \xi_n}{n}}
    \right)
.\end{equation}
But the hidden constants and logarithmic factors in \eqref{eq:opt-rate-linear} do not meet our needs:
to show consistency (as we discuss shortly)
we need an asymptotic coefficient of 1 on $B^2 \xi_n / n$,
while \cite{optimistic-rates} showed only an upper bound of $200\,000 \log^3(n)$.
It seems likely given their extremely indirect proof technique, though, that a much tighter version holds --
especially in the special case of bounded-norm linear predictors for square loss.
Given \cref{prop:wmn-size},
it is reasonable to suspect that something like the following may hold fairly generally:
\begin{equation} \label{eq:spec-ub}
    \sup_{\norm w \le B, \, L_\samp(w) = 0}
        L_\D(w) - L_\samp(w)
    \le \frac1n B^2 \xi_n + o_P(1)
    \tag{$\star$}
,\end{equation}
where here $\xi_n$ might refer either to the high-probability upper bound on $\norm x^2$
or, for sub-Gaussian data, perhaps simply $\E \norm x^2$.
For either choice of $\xi_n$,\footnote{
    If $\xi_n$ is a high-probability upper bound, we further require $\lambda_n = \omega(\log n)$.}
by taking $B = \norm \wmn$
in \cref{setting:junk-feats},
applying \cref{prop:wmn-size} then gives us (subject to integrability conditions)
that
for $\lambda_n = \omega(1)$, $\lambda_n = o(n)$,
\begin{equation}
\lim_{d_J \to \infty} \E L_\D(\wmn)
=
\lim_{d_J \to \infty} \E\left[
    \sup_{\norm w \le \norm{\wmn}, \, L_\samp(w) = 0} L_\D(w) - L_\samp(w)
\right] \le \sigma^2 + o(1)
.\end{equation}
But \eqref{eq:spec-ub} would also do more than this:
it makes predictions about the generalization error of interpolators with larger-than-minimal norm,
not yet known in the literature.
In the setting of \cref{prop:wmn-size},
\eqref{eq:spec-ub} would imply that
\begin{equation} \label{eq:spec-ub-preds}
    \lim_{d_J \to \infty}
    \E\left[ \sup_{\norm{w} \le \alpha \norm{\wmn}, L_\samp(w) = 0}  L_\D(w) - L_\samp(w) \right]
    \le \alpha^2 \left[ \sigma^2 + o(1) \right]
.\end{equation}
These predictions are important in their own right:
outside of linear models,
we rarely expect to obtain the interpolator with \emph{exactly} minimal norm.

\subsection{Uniform convergence of low-norm interpolators in Setting \ref{setting:junk-feats}} \label{sec:non-min-norm}

The predictions made in \eqref{eq:spec-ub-preds} in fact hold, with equality.
\begin{restatable}{theorem}{junkball} \label{thm:junk-ball}
    In \cref{setting:junk-feats} with $\lambda_n = o(n)$,
    fix a sequence $(\alpha_n) \to \alpha$,
    with each $\alpha_n \ge 1$.
    Then
    \[
        \lim_{n \to \infty} \lim_{d_J \to \infty}
        \E\left[ \sup_{\norm{w} \le \alpha_n \norm\wmn, \, L_\samp(w) = 0} L_\D(w) - L_\samp(w) \right] = \alpha^2 L_{\D}(w^*)
    .\]
\end{restatable}
The proof of \cref{thm:junk-ball} is based on bounding \eqref{eq:gen-gap} directly,
although it will take us several steps to get there which we now outline.
Along the way, we provide results, especially \cref{prop:wmr-consistent}, which are applicable well beyond \cref{setting:junk-feats}.

The first tool we will require in our analysis is the best-conceivable interpolator for a given $X$ and $\D$:
\begin{defn}
The \emph{minimal-risk interpolator} \citep[Section 3.3]{muthukumar:interpolation} is
\begin{equation}
        \wmr
         = \argmin_{w \st Xw = Y} \, L_\D(w)
         = w^* + \Sigma^{-1} X\tp (X \Sigma^{-1} X\tp)^{-1} E  
\label{eq:wmr}
.\end{equation}
\end{defn}

\begin{restatable}{prop}{wmrconsistent} \label{prop:wmr-consistent}
In \cref{setting:lin-gauss}, the expected risk of the minimal-risk interpolator is
\[
    \E L_\D(\wmr) =  \frac{p-1}{p-1-n} L_{\D}(w^*)
.\]
\end{restatable}
Because $\wmr$ has perfect knowledge of $\Sigma$,
its expected risk turns out to be independent of $\Sigma$.
As $p$ increases for fixed $n$ (the second of the double descents), $\E L_\D(\wmr)$ thus improves monotonically:
$\wmr$ can pick among more interpolators.

We use $\wmr$ as a constructive tool in our proofs:
\cref{thm:general-consistency}
expands the generalization gap around a fixed predictor in terms of that predictor's risk,
and so the minimal-risk predictor is an obvious choice for understanding the gap.
\Cref{prop:wmr-consistent} also provides lower bounds on interpolation methods:
if $p = \mathcal O(n)$, then $\wmr$ is not consistent, and hence no interpolator is.
For instance,
LASSO is minimax-optimal and consistent for sparse linear regression when $n = \Theta(p)$ \citep{lasso,consistency-lasso,model-consistency-lasso,dantzig,high-dimension,minimax-lasso},
but no interpolation method can be.
\citet[Section 3]{muthukumar:interpolation} discuss this type of result in detail, including for non-Gaussian data;
see also \cite{JLL:basis-pursuit}.

Our next tool measures how much energy in $\Sigma$ is missed by the sample $X$.
\begin{defn} \label{def:restr-eig}
The \emph{restricted eigenvalue under interpolation}
for covariance $\Sigma$ and design $X$
is
\[
    \kappa_X(\Sigma) = \sup_{\norm w = 1,\; X w = 0} w\tp \Sigma w 
.\]
\end{defn}

We now have the tools to show the following result,
which holds even more generally than \cref{setting:lin-gauss}.

\begin{restatable}{theorem}{generalconsistency} \label{thm:general-consistency}
    The following results hold deterministically,
    viewing $L_\D(w)$ simply as a quadratic function
    $L_\D(w^*) + \norm{w - w^*}_\Sigma$,
    with no distributional assumptions on $\samp$.
\begin{enumerate}[label=(\roman*)]
    \item \label{thm-part:general-consistency:wmr}
    It holds that
    \[
        \sup_{\substack{\norm{w} \le \norm{\wmr} \\ L_\samp(w) = 0}} L_\D(w) - L_\samp(w) = L_\D(\wmr) + \gamma_n \, \kappa_X(\Sigma) \, \Big[ \norm{\wmr}^2 - \norm{\wmn}^2 \Big]
    \]
    where $1 \le \gamma_n \le 4$.
    
    If the minimal risk interpolator is consistent,
    $\E L_\D(\wmr) - L_\D(w^*) \to 0$,
    then the class of interpolators with norm less than $\norm\wmr$ is uniformly consistent if and only if 
    \[
        \E \kappa_X(\Sigma) \cdot \Big[ \norm\wmr^2 - \norm\wmn^2 \Big] \to 0
    .\] 
    
    \item \label{thm-part:general-consistency:ball}
        Fix a sequence $(B_n)$ such that $B_n \geq \norm{\wmn}$ for all $n$. Then
        \[
            \sup_{\substack{\norm w \le B_n, \, L_\samp(w) = 0}} L_\D(w) - L_\samp(w)
            = L_\D(\wmn) + \kappa_X(\Sigma) \left[ B_n^2 - \norm{\wmn}^2 \right] + R_n
        \]
        where
        $0 \le R_n \le 2 \sqrt{
            \left[ L_\D(\wmn) - L_\D(w^*) \right]
            \kappa_X(\Sigma)
            \left[ B_n^2 - \norm{\wmn}^2 \right]
        }$.
        
        If $\E L_\D(\wmn) - L_\D(w^*) \to 0$,
        the class of interpolators with norm less than $B_n$ is thus uniformly consistent if and only if 
        \[
            \E \kappa_X(\Sigma) \cdot \Big[ B_n^2 - \norm{\wmn}^2 \Big] \to 0
        .\]
\end{enumerate}
\end{restatable}

The term $\kappa_X(\Sigma) [B^2 - \norm\wmn^2]$ appearing in each bound
multiplies $\kappa$, essentially ``how much'' of $\Sigma$ is orthogonal to the data sample,
by the amount of excess norm available inside the norm ball.
This result makes us expect that \eqref{eq:spec-ub}
should in fact hold fairly generally with $\xi_n = n \, \kappa_X(\Sigma)$.

Notice also that, of course, $\norm\wmn \le \norm\wmr$;
thus when $\wmr$ is consistent (e.g.\ via \cref{prop:wmr-consistent})
and $\E \kappa_X(\Sigma) [\norm\wmr^2 - \norm\wmn^2] \to 0$,
then \ref{thm-part:general-consistency:wmr} implies $\wmn$ is consistent as well.

\begin{proof}[Proof sketch]
Let $\hat{w}$ be any particular predictor that interpolates the data, and $F \in \R^{p \times (p-n)}$  be the matrix whose columns form an orthonormal basis of the kernel of $X$.
Then \eqref{eq:gen-gap} can be rewritten as
\begin{equation}  \label{eq:gen-gap-qcqp}
    \sup_{u \in \R^{p-n}: \norm{ \hat{w} + F u }^2 \leq B^2 } \norm{\hat w + F u - w^*}^2_\Sigma
.\end{equation}
This is a quadratic program with a single quadratic constraint,
which enjoys strong duality even though it is a convex \emph{maximization} \citep[Appendix B]{convex-optimization}.
We thus need analyze only the (much simpler) one-dimensional dual problem.
For \ref{thm-part:general-consistency:ball},
we take $\hat{w} = \wmn$ in \eqref{eq:gen-gap-qcqp} and obtain the dual as
\[
    \inf_{\lambda > \norm{F\tp \Sigma F}}
        L_\D(\wmn)
        + \norm{ F\tp \Sigma(\wmn - w^*)}_{(\lambda I_{p-n} - F\tp \Sigma F)^{-1}}^2
        + \lambda \Big[ B_n^2 - \norm\wmn^2 \Big]
.\]
Given consistency, we can show that the second term's contribution is negligible, as 
\[
    \norm{F \tp \Sigma(\wmn - w^*)}^2
    \leq \norm{F\tp \Sigma F} \cdot [L_\D(\wmn) - L_\D(w^*)]
,\]
and $(\lambda I_{p-n} - F\tp \Sigma F)^{-1}$ has controlled eigenvalues so that the Mahalanobis norm is similar to the Euclidean norm.
Observing that $\kappa_X(\Sigma) = \norm{F\tp \Sigma F} $, the conclusion follows by routine calculations.

Case \ref{thm-part:general-consistency:wmr} uses a similar strategy, taking $\hat w = \wmr$.
The full proof is given in \cref{sec:proofs:non-min-norm}.
\end{proof}

Now, all that remains is to evaluate the relevant quantities in \cref{setting:junk-feats}.
\begin{restatable}{prop}{junkwmr} \label{thm:junk-wmr}
    In \cref{setting:junk-feats} with $\lambda_n = o(n)$,
    \[
        \lim_{n \to \infty} \lim_{d_J \to \infty} \E\left[
            \sup_{\norm w \le \norm\wmr,\, L_\samp(w) = 0} L_\D(w) - L_\samp(w) 
        \right]
        = L_{\D}(w^*)
    .\]
\end{restatable}

\begin{proof}[Proof sketch for \cref{thm:junk-ball,thm:junk-wmr}] 
We apply \cref{thm:general-consistency}.
With probability one,
\[
    \lim_{d_J \to \infty} \kappa_X(\Sigma)
    = \frac{\lambda_n}{n} \norm*{ \left[
        \frac{X_S\tp X_S}{n} + \frac{\lambda_n}{n} I_{d_S}
    \right]^{-1}  }
.\]
As the first term inside the inverse converges to $I_{d_S}$ and the second term vanishes, we can expect $\kappa_X(\Sigma) \approx {\lambda_n}/{n}$. 
We bound the other terms by observing that
there exists a sequence $\beta_n \to 1$ with
\begin{gather*}
    \lim_{d_J \to \infty} \E \norm{\wmr}^2
    = \norm{w_S^*}^2 + \frac{\sigma^2 n}{\lambda_n}
    \\
    \lim_{d_J \to \infty} \E \norm{\wmn}^2
    = \norm{w^*}^2 + \sigma^2 \frac{n-d_S}{\lambda_n} + \beta_n \left( \frac{\sigma^2 d_S - \lambda_n \norm{w_S^*}^2 }{n} \right)
,\end{gather*}
so
$
    \lim_{d_J \to \infty} \E\left[
        \norm{\wmr}^2 - \E \norm{\wmn}^2
    \right]
    = {\sigma^2 d_S}/{\lambda_n} + \mathcal{O}\left( {\lambda_n \norm{w^*}^2 }/{n} \right)
$.

Because $\wmr$ is consistent via \cref{prop:wmr-consistent},
this proves \cref{thm:junk-wmr}.
As $\norm\wmn \le \norm\wmr$,
this further implies $\wmn$ is consistent,
so that the $R_n$ term of \cref{thm:general-consistency} \ref{thm-part:general-consistency:ball} vanishes.
\end{proof}
We can see that $\kappa_X(\Sigma)$ tends to 0 while $\norm{\wmn}$ explodes,
and in \cref{setting:junk-feats}
their product turns out to converge to \emph{exactly} the Bayes risk.
Because the other terms of \cref{thm:general-consistency} \ref{thm-part:general-consistency:ball} cancel,
this gives us precisely the tight result we need for \cref{thm:junk-ball},
and further suggests that the speculative upper bound $\kappa_X(\Sigma) B^2$ probably holds in more general settings.

We have at last shown in \cref{thm:junk-ball} a uniform convergence bound not only showing consistency of $\wmn$,
but furthermore verifying the predictions of \eqref{eq:spec-ub-preds}.
Thus if we obtain an interpolator with norm
$1.1 \norm\wmn$, we will suffer at most $1.21 \sigma^2$ asymptotic risk.
If we obtain an interpolator with norm no more than a constant amount larger than the minimal norm,
we achieve asymptotic consistency.

\section{Discussion}

In this work, we shed new light on uniform convergence and its relationship to interpolation learning.  We show that uniform control of the generalization gap cannot explain interpolation learning, for almost \emph{any} interpolator, even in a simple setting.  But we argue that when discussing ``uniform convergence'' in the context of interpolation learning, we should slightly broaden our horizons to include
interpolation-specific uniform convergence bounds such as \eqref{eq:spec-ub}, or more generally ``optimistic'' (training-error-dependent) bounds \citep{panchenkooptimistic,optimistic-rates}.
We show that despite recent sentiments to the contrary, such bounds {\em could} in principal explain interpolation learning, by demonstrating this in the ``junk features'' setting.
Doing so requires obtaining very tight bounds, include tight constants -- perhaps a difficult task, but not impossible.
(For example, for linear predictors with a Lipschitz loss in a non-realizable setting, we do know the exact worst-case bound, with a tight numeric constant \citep{kakade2009complexity}.)

Our results are also of independent interest in ensuring success with interpolation learning:
in settings other than linear regression, where a closed-form solution is available,
it is generally unlikely in practice that we find the \emph{exact} minimum-norm solution.
(Even gradient descent for linear regression would find this only when initialized exactly in the span of the data; other forms of implicit bias are likewise suboptimal.)
Our results give some reassurance that, at least in this simple setting,
approximately minimizing the norm is sufficient.
The natural next step in this vein would be to study predictors with small but nonzero loss.
This could either be done directly in the style of our \cref{thm:junk-ball},
or by providing an optimistic rate as in \eqref{eq:opt-rate-linear} with tight constants.
Our specific techniques,
as well as the general takeaway of considering interpolation-specific bounds,
could also be potentially applicable to settings beyond linear regression,
especially the idea of studying the generalization gap via the dual problem:
although strong duality may not be available in more general settings,
upper bounds are always possible with weak duality.

\section*{Broader Impact}
Interpolation learning is currently thought to be one of the core mysteries standing between us and a theoretical understanding of modern deep learning.
Although there has recently been some key progress, many challenges remain.
Our paper, in advancing the study of interpolation learning,
makes another step on the path towards understanding the deep learning models that are quickly becoming ubiquitous throughout society, whether we understand them or not.  In our view, increased understanding of these models can lead to safer, more reliable, and more controlled deployment, especially in sensitive domains.

In particular, we discuss a key component of statistical learning theory, namely uniform convergence, whose relevance to deep learning in general -- and interpolation learning specifically -- has recently been questioned.  We make an explicit connection between the work on interpolation learning and the recent notion of ``algorithmic dependent uniform convergence'' \citep{nagarajan:uniform}.  Instead of outright dismissal, we show that a more nuanced view is appropriate.  By doing so, we hope to help guide the re-pivoting that statistical learning theory is currently undergoing.

We emphasize that, despite providing some positive theoretical results,
we are certainly not advocating for preferring interpolation methods over other approaches. In particular, the increased sensitivity of interpolation methods may have problematic ramifications for robustness or privacy.

\begin{ack}
Research supported in part by NSF IIS award 1764032 and NSF HDR TRIPODS award 1934843.
\end{ack}

\printbibliography

\clearpage\appendix

\section{Proofs for Section \ref{sec:problem-setting}}

\ridgeconsistent*
\begin{proof}
We can write
\begin{equation*}
    \begin{split}
        \hat{w}_{\lambda_n} -  w_S^* &= (X_S\tp X_S + \lambda_n I_{d_S})^{-1} X_S\tp (X_S w_S^* + E) -  w_S^*\\
        &= ((X_S\tp X_S + \lambda_n I_{d_S})^{-1} X_S\tp X_S - I_{d_S})  w_S^* + (X_S\tp X_S + \lambda_n I_{d_S})^{-1} X_S\tp E\\
        &= \left[ \left(\frac{X_S\tp X_S}{n}  + \frac{\lambda_n}{n} I_{d_S}\right)^{-1} \frac{X_S\tp X_S}{n} - I_{d_S}\right]  w_S^* + \left(\frac{X_S\tp X_S}{n} + \frac{\lambda_n}{n} I_{d_S}\right)^{-1} \frac{X_S\tp E}{n}
    .\end{split}
\end{equation*}
Therefore, by independence of $X_S$ and $E$,
\begin{equation*}
    \begin{split}
        \E &[L_{\D}(\hat{w}_{\lambda_n}) - L_{\D}(w^*)]
         = \E \norm{\hat{w}_{\lambda_n} -  w_S^*}^2 \\
        &= \E \, \Bigg\lVert \left[ \left(\frac{X_S\tp X_S}{n}  + \frac{\lambda_n}{n} I_{d_S}\right)^{-1} \frac{X_S\tp X_S}{n} - I_{d_S}\right]  w_S^* \Bigg\rVert^2
         + \E  \, \Bigg\lVert \left(\frac{X_S\tp X_S}{n} + \frac{\lambda_n}{n} I_{d_S}\right)^{-1} \frac{X_S\tp E}{n}  \Bigg\rVert^2 \\
        &= \E \, \Bigg\lVert \left[ \left(\frac{X_S\tp X_S}{n}  + \frac{\lambda_n}{n} I_{d_S}\right)^{-1} \frac{X_S\tp X_S}{n} - I_{d_S}\right]  w_S^* \Bigg\rVert^2
          + \sigma^2 \E \frac{1}{n} \tr \left[ \left(\frac{X_S\tp X_S}{n} + \frac{\lambda_n}{n} I_{d_S}\right)^{-2} \frac{X_S\tp X_S}{n}  \right]
    .\end{split}
\end{equation*}

Write the SVD for $X_S = U D V\tp$.
Since $X_S$ has rank at most $d_S$, we denote its singular values as $\sqrt{\rho_1}, ..., \sqrt{\rho_{d_S}}$, and
\[
    \norm{(X_S\tp X_S+ \lambda I_{d_S})^{-1} X_S\tp X_S}
    = \norm{(D\tp D + \lambda I_{d_S})^{-1} D\tp D}
    = \max_{i \in [p]} \frac{\rho_i}{\lambda_n + \rho_i}
    \leq 1
.\]
Thus, we have
\[
    \norm*{ \left[ \left(\frac{X_S\tp X_S}{n}  + \frac{\lambda}{n} I_{d_S}\right)^{-1} \frac{X_S\tp X_S}{n} - I_{d_S}\right]  w_S^* }^2
    \leq (1 + 1)^2 \norm{w_S^*}^2
    = 4 \norm{w_S^*}^2
\]
which is clearly integrable. 

As $d_S$ stays fixed as $n \to \infty$, by the strong law of large numbers we have $\frac{X_S\tp X_S}{n} \to I_{d_S}$.
Assuming that $\frac{\lambda_n}{n} \to \gamma$, then by the continuous mapping and dominated convergence theorems, the first term converges to
\begin{equation*}
    \begin{split}
         \E \lim_{n \to \infty} \norm*{ \left[ 1 - \left(1  + \gamma \right)^{-1}  \right]  w_S^* }^2 &= \left( \frac{\gamma}{1+\gamma} \cdot \norm{w_S^*} \right)^2
    ,\end{split}
\end{equation*}

Moreover, it holds that
\begin{equation*}
    \begin{split}
         \, \frac{1}{n} \tr \left[ \left(\frac{X_S\tp X_S}{n} + \frac{\lambda_n}{n} I_{d_S}\right)^{-2} \frac{X_S\tp X_S}{n}  \right] 
         & = \sum_{i=1}^{d_S} \left( \frac{\sqrt{\rho_i}}{\rho_i+\lambda_n} \right)^2 \\
         & \leq \sum_{i=1}^{d_S} \frac{1}{\rho_i} = \tr \left[ \left( X_S\tp X_S \right)^{-1} \right]
    \end{split}
\end{equation*}
Using the first moment of inverse Wishart distribution, the second term can be controlled by
\[ 
\sigma^2 \E \tr \left[ \left( X_S\tp X_S \right)^{-1} \right] = \sigma^2 \frac{d_S}{n- d_S - 1} \to 0
\]
Note that the first term converges to 0 as long as $\gamma = 0$, and the desired conclusion follows.
\end{proof}

\section{Proofs for Section \ref{sec:uniform-gen-gap}}

\subsection{Size of the minimal-norm interpolator\texorpdfstring{ (\cref{prop:wmn-size})}{}}

\begin{prop} \label{thm:mn-mr-norms}
In \cref{setting:junk-feats}, it holds that
\[
  \lim_{d_J \to \infty} \E \norm{\wmr}^2 = \norm{w^*}^2 +  \frac{\sigma^2n}{\lambda_n}
.\]
Moreover, there exists a sequence $( \beta_n )$ such that $\beta_n \to 1$ and 
\[
  \lim_{d_J \to \infty} \E \norm{\wmn}^2
  = \norm{w^*}^2 + \sigma^2 \frac{n-d_S}{\lambda_n}
  + \beta_n \left( \frac{\sigma^2 d_S - \lambda_n \norm{w_S^*}^2 }{n} \right)
.\] 
Consequently, we have
\[
  \lim_{d_J \to \infty} \E \left[ \norm{\wmr}^2 - \norm{\wmn}^2 \right]
  = \frac{\sigma^2 d_S}{\lambda_n}
  + \beta_n \left( \frac{\lambda_n \norm{w_S^*}^2 -\sigma^2 d_S }{n} \right)
.\]
\end{prop}

\begin{proof}
Let $\{ e_i \}$ be the standard basis in $\R^p$
and write $ \Sigma = \sum_{i=1}^p \mu_i e_i e_i^T $,
with $\mu_i = 1$ for $1 \le i \le d_S$ and $\mu_i = {\lambda_n}/{d_J}$ for $i > d_S$.
By independence of $X$ and $E$, we have
\begin{equation*}
    \begin{split}
        \E \norm{\wmr}^2 &=  \norm{w^*}^2 + \E \norm{\Sigma^{-1} X\tp (X \Sigma^{-1} X\tp)^{-1} E}^2 \\
        &= \norm{w^*}^2 + \sigma^2 \E \left[  \tr \Big( (ZZ\tp)^{-1} (Z \Sigma^{-1} Z\tp) (ZZ\tp)^{-1} \Big) \right] \\
        &= \norm{w^*}^2 + \sum_{i=1}^p \frac{\sigma^2 }{\mu_i} \E \left[ \norm{ (ZZ^T)^{-1} Z e_i }^2 \right]
    .\end{split}
\end{equation*}
By rotational invariance of the standard normal distribution for $Z$, we have
\begin{equation*}
  \E \left[ \norm{ (ZZ^T)^{-1} Z e_i }^2 \right]
   = \frac{\E \tr( Z^T (ZZ^T)^{-2} Z) }{p}
   = \frac{\E \tr( (ZZ^T)^{-1} ) }{p}
   = \frac{n}{p(p - n -1)}
.\end{equation*}
Plugging in, we get
\begin{equation*}
    \begin{split}
        \E \norm{\wmr}^2
        &= \norm{w^*}^2 + \left( \sum_{i=1}^p \frac{\sigma^2 }{\mu_i} \right) \frac{n}{p(p - n -1)} \\
         &= \norm{w^*}^2 + \sigma^2  \left( d_S + \frac{d_J^2}{\lambda_n} \right) \frac{n}{p(p - n -1)}
    .\end{split}
\end{equation*}
Sending $d_J \to \infty$ and recalling $p = d_S + d_J$, we obtain
\[ \lim_{d_J \to \infty} \E \norm{\wmr}^2 = \norm{w^*}^2 +  \frac{\sigma^2n}{\lambda_n} .\]
Moreover, it holds that
\begin{equation*}
    \begin{split}
        \norm{\wmr}^2 &= \norm{w^*}^2 +  \tr \Big( (ZZ\tp)^{-1} (Z \Sigma^{-1} Z\tp) (ZZ\tp)^{-1} EE\tp \Big) + 2 \langle  w^*, \Sigma^{-1/2} Z\tp (ZZ\tp)^{-1} E \rangle\\
        &= \norm{w^*}^2
         +  \tr \Bigg( \left(\frac{ZZ\tp}{p}\right)^{-1} \left(\frac{Z \Sigma^{-1} Z\tp}{p^2}\right) \left(\frac{ZZ\tp}{p}\right)^{-1} EE\tp \Bigg)
         + 2 \inner*{ \frac{Z\Sigma^{-1/2} w^* E\tp}{p} }{  \left(\frac{ZZ\tp}{p}\right)^{-1} }
    .\end{split}
\end{equation*}
Notice that
\begin{gather*}
    \lim_{d_J \to \infty } \left(\frac{ZZ\tp}{p}\right)^{-1}
    \aseq I_{n}
\\
    \lim_{d_J \to \infty } \frac{Z \Sigma^{-1} Z\tp}{p^2}
    = \lim_{d_J \to \infty } \frac{1}{p^2}
        \left( Z_S Z_S\tp + \frac{d_J^2}{\lambda_n} \frac{Z_J Z_J\tp}{d_J} \right)
    \aseq \frac{1}{\lambda_n} I_n
\\
    Z\Sigma^{-1/2} w^* E\tp
    = \begin{bmatrix}Z_S & Z_J \end{bmatrix}
    \begin{bmatrix}
        I_{d_S} & 0_{d_S \times d_J} \\
        0_{d_J \times d_S} & \sqrt{\frac{d_J}{\lambda_n}} I_{d_J}
    \end{bmatrix} 
    \begin{bmatrix}
        w^*_S \\
        0_{d_J}
    \end{bmatrix}
    E\tp
    = Z_S w_S^* E\tp 
\implies
    \frac{Z\Sigma^{-1/2} w^* E\tp}{p} \aseq 0
.\end{gather*}
Plugging in, we obtain
\[
  \lim_{d_J \to \infty} \norm{\wmr}^2
  \aseq \norm{w^*}^2 + \frac{\norm{E}^2}{\lambda_n}
  ,\quad\text{and so}\quad
  \E \left[ \lim_{d_J \to \infty} \norm{\wmr}^2 \right]
  =  \lim_{d_J \to \infty} \E \norm{\wmr}^2
.\] 
Clearly, the sequence of random variables $( \norm{\wmr}^2 )$ as we let $d_J \to \infty$ dominates $( \norm{\wmn}^2 )$.
By the dominated convergence theorem \footnote{%
    We use the following version of the theorem, which is slightly more general than the usual one.
    Suppose there exists a sequence of $l_1$ random variables $Y_n$ such that $Y_n \geq X_n$ and
    \[ \lim_{n \to \infty} \E \, Y_n = \E \lim_{n \to \infty} Y_n ;\] 
    then we have
    \[ \lim_{n \to \infty} \E \, X_n = \E \lim_{n \to \infty} X_n .\] 
    The proof is essentially the same and applies Fatou's lemma to $X_n$ and $Y_n - X_n$.
}
\begin{equation*}
    \begin{split}
        \lim_{d_J \to \infty} \E \norm{\wmn}^2 &= \E \left[ \lim_{d_J \to \infty}  \norm{\wmn}^2 \right] \\
        &= \E \Big[  \lim_{d_J \to \infty} (X_S w_S^* + E)\tp  (XX\tp)^{-1} X X\tp (XX\tp)^{-1} (X_S w_S^* + E) \Big] \\
        &= \E \Big[  \lim_{d_J \to \infty} (X_S w_S^* + E)\tp  (X_S X_S\tp + X_J X_J\tp)^{-1} (X_S w_S^* + E) \Big] \\
        &= \E \Big[  (X_S w_S^* + E)\tp  (X_S X_S\tp +\lambda_n I_n)^{-1} (X_S w_S^* + E) \Big] \\
        &= (w_S^*)\tp \E [X_S\tp (X_S X_S\tp + \lambda_n I_n)^{-1} X_S] w_S^* + \sigma^2 \, \E \tr \left( (X_S X_S\tp + \lambda_n I_n)^{-1} \right)
    .\end{split}
\end{equation*}
With probability one, $X_S X_S\tp$ is a $n \times n$ matrix with rank $d_S$, so the eigenvalues of $(X_S X_S\tp + \lambda_n I_n)^{-1}$ consist of the $d_S$ eigenvalues of $(X_S\tp X_S + \lambda_n I_{d_S})^{-1}$ and $(n-d_S)$ copies of $\frac{1}{0 + \lambda_n}$.
This implies
\[
  \sigma^2 \, \E \tr \left( (X_S X_S\tp + \lambda I_n)^{-1} \right)
  =  \sigma^2 \E \tr \left( (X_S\tp X_S + \lambda I_{d_S})^{-1} \right) + \sigma^2 \frac{n-d_S}{\lambda_n}
.\] 
Moreover, by the rotational invariance of $X_S \sim \N(0, I_{d_S})$,
\begin{equation*}
    \begin{split}
         (w_S^*)\tp \E [X_S\tp (X_S X_S\tp + \lambda_n I_n)^{-1} X_S] w_S^* &= \frac{\norm{w_S^*}^2}{d_S} \E \tr \left (X_S\tp (X_S X_S\tp + \lambda_n I_n)^{-1} X_S  \right) \\
         &= \frac{\norm{w_S^*}^2}{d_S} \E \tr \left (X_S\tp X_S (X_S\tp X_S + \lambda_n I_{d_S})^{-1}  \right) \\
         &= \frac{\norm{w_S^*}^2}{d_S} \E \tr \left (I_{d_S} - \lambda_n (X_S\tp X_S + \lambda_n I_{d_S})^{-1}  \right) \\
         &= \norm{w_S^*}^2 - \frac{\lambda_n \norm{w_S^*}^2}{d_S} \E \tr \left ((X_S\tp X_S + \lambda_n I_{d_S})^{-1}  \right)
    .\end{split}
\end{equation*}
Plugging in, we get
\begin{equation*}
    \begin{split}
        \lim_{d_J \to \infty} \E \norm{\wmn}^2 &= \norm{w^*}^2 + \sigma^2 \frac{n-d_S}{\lambda_n} + \left( \sigma^2 - \frac{\lambda_n \norm{w_S^*}^2}{d_S} \right) \E \tr \left ((X_S\tp X_S + \lambda_n I_{d_S})^{-1}  \right)\\
         &= \norm{w^*}^2 + \sigma^2 \frac{n-d_S}{\lambda_n} + \left( \frac{\sigma^2 d_S - \lambda_n \norm{w_S^*}^2 }{n} \right) \cdot \left[ \frac{\E \tr \left ( \left( \frac{X_S\tp X_S}{n} + \frac{\lambda_n}{n} I_{d_S} \right)^{-1}  \right)}{d_S} \right]
    .\end{split}
\end{equation*}
As $\tr \left ( \left( \frac{X_S\tp X_S}{n} \right)^{-1} \right)$,
which has limit $d_S$ in expectation,\footnote{%
    Using standard properties of the inverse Wishart distribution, we can check that
    \[ \lim_{n \to \infty} \E \tr \left ( \left( \frac{X_S\tp X_S}{n} \right)^{-1} \right) = d_S =  \E \lim_{n \to \infty} \tr \left ( \left( \frac{X_S\tp X_S}{n} \right)^{-1} \right) .\] 
}
dominates $\tr \left ( \left( \frac{X_S\tp X_S}{n} + \frac{\lambda_n}{n} I_{d_S} \right)^{-1}  \right)$,
by the dominated convergence theorem
\[
    \lim_{n \to \infty} \frac{1}{d_S} \E \tr \left( \left( \frac{X_S\tp X_S}{n} + \frac{\lambda_n}{n} I_{d_S} \right)^{-1}  \right) = 1
.\] 
Letting the term in brackets be $\beta_n$, we have the result.
\end{proof}

\wmnsizebasic*
\begin{proof}
By \cref{thm:mn-mr-norms}, there exists a sequence $( \beta_n )$ such that $\beta_n \to 1$ and 
\[ \lim_{d_J \to \infty} \E \norm{\wmn}^2 = \sigma^2 \frac{n}{\lambda_n} + \left[  \norm{w^*}^2 - \sigma^2 \frac{d_S}{\lambda_n} + \beta_n \left( \frac{\sigma^2 d_S - \lambda_n \norm{w_S^*}^2 }{n} \right) \right] .\] 
Moreover, we have
\[ \E \norm{x}^2 = \tr(\Sigma) = d_S \cdot 1 + d_J \cdot \frac{\lambda_n}{d_J} = d_S + \lambda_n .\]
Plugging in, we obtain
\[ \frac{(\E \norm{\wmn}^2)( \E \norm{x}^2)}{n} = \sigma^2 \frac{d_S + \lambda_n}{\lambda_n} + \frac{d_S + \lambda_n}{n} \left[  \norm{w^*}^2 - \sigma^2 \frac{d_S}{\lambda_n} + \beta_n \left( \frac{\sigma^2 d_S - \lambda_n \norm{w_S^*}^2 }{n} \right) \right] .\]
By assumption, $1 / \lambda_n \to 0$
and $\lambda_n / n \to 0$;
thus the dominant term inside the brackets
is $ \norm{w^*}^2 = \mathcal{O}(1)$. The conclusion follows by
\[
    \frac{d_S + \lambda_n}{\lambda_n} \to 1
    \quad\text{and}\quad
    \frac{d_S + \lambda_n}{n} \to 0
.\qedhere \]
\end{proof}

\subsection{Divergence of the generalization gap of norm balls (Section \ref{sec:one-sided}) } \label{sec:proofs:ball-gap}

\begin{prop} \label{thm:ballgapbound}
Let $\rho(\Sigma - \hat{\Sigma})$ be the algebraically largest eigenvalue of $\Sigma - \hat{\Sigma}$. It holds that
\[
    \sup_{\norm{w} \leq \norm{\wmn} }  L_\D(w) - L_\samp(w) \geq
    \rho(\Sigma - \hat{\Sigma}) \cdot (\norm{\wmn} - \norm{w^*} )^2 + \left[ L_{\D} (w^*) - \frac1n \norm{E}^2 \right]
\]
and similarly for two sided uniform convergence, it holds that 
\[
    \sup_{\norm{w} \leq \norm{\wmn} }  \abs{ L_\D(w) - L_\samp(w) } \geq
    \norm{\Sigma - \hat{\Sigma}} \cdot (\norm{\wmn} - \norm{w^*} )^2 - \Big|  L_{\D} (w^*) - \frac{ \norm{E}^2 }{n} \Big|
.\]
\end{prop}

\begin{proof}
Recall from \eqref{eq:ld-and-ls} that
\begin{equation*}
    \begin{split}
        L_{\samp} (w)
        &= \frac1n \norm{X w - Y}^2 \\
        &= \frac1n \norm{X (w - w^*) + X w^* - Y}^2 \\
        &= (w - w^*)\tp \hat{\Sigma} (w - w^*) + \frac{\norm{E}^2}{n} - 2 \Big\langle w - w^*, \frac{X\tp E}{n} \Big\rangle
    .\end{split}
\end{equation*}

Therefore, we can decompose the generalization gap as
\begin{equation*}
    \begin{split}
        L_{\D} (w) - L_{\samp} (w) &= L_{\D} (w^*) + (w - w^*)\tp \Sigma (w - w^*) - L_{\samp} (w) \\
        &= \left[ L_{\D} (w^*) - \frac{\norm{E}^2}{n} \right] + (w - w^*)\tp (\Sigma - \hat{\Sigma}) (w - w^*) + 2 \Big\langle w - w^*, \frac{X\tp E}{n} \Big\rangle
    .\end{split}
\end{equation*}
Observe that
\begin{equation*}
    \begin{split}
        \sup_{\norm{w} \leq \norm{\wmn}}  (w - w^*)\tp (\Sigma - \hat{\Sigma}) (w - w^*) &+ 2 \Big\langle w - w^*, \frac{X\tp E}{n} \Big\rangle \\
    \geq \, \, &\sup_{\norm{w} \leq \norm{\wmn} - \norm{w^*} }  w\tp (\Sigma - \hat{\Sigma}) w + 2 \Big\langle w, \frac{X\tp E}{n} \Big\rangle \\
    \geq \, \, &\rho(\Sigma - \hat{\Sigma}) \cdot (\norm{\wmn} - \norm{w^*} )^2
    .\end{split}
\end{equation*}
The last inequality holds by picking $w$ to be
$\pm (\norm\wmn - \norm{w^*})$ times the top eigenvector of $\Sigma - \hat\Sigma$
for whichever sign makes the linear term nonnegative.
By the same reasoning, we have 
\[
    \sup_{\norm{w} \leq \norm{\wmn} }  \abs{ L_\D(w) - L_\samp(w) } \geq
    \norm{\Sigma - \hat{\Sigma}} \cdot (\norm{\wmn} - \norm{w^*} )^2 - \Big|  L_{\D} (w^*) - \frac{ \norm{E}^2 }{n} \Big|
. \qedhere \]
\end{proof}

\ballgapdiverges*
\begin{proof}
We will show that in \cref{setting:junk-feats} as long as $\lambda_n = o(n)$,
\[
    \lim_{n \to \infty} \lim_{d_J \to \infty} \E \norm{\Sigma - \hat{\Sigma}} \cdot \norm{\wmn}^2 = \infty
.\]
By Fatou's lemma and the calculation in \cref{thm:mn-mr-norms},
\begin{equation*}
    \begin{split}
        \lim_{d_J \to \infty} \E \norm{\Sigma - \hat{\Sigma}} \cdot \norm{\wmn}^2 &\geq  \E \lim_{d_J \to \infty} \norm{\Sigma - \hat{\Sigma}} \cdot \norm{\wmn}^2\\
        &=  \E \lim_{d_J \to \infty} \norm{\Sigma - \hat{\Sigma}} \cdot \left( (X_S w_S^* + E)\tp  (X_S X_S\tp +\lambda_n I_n)^{-1} (X_S w_S^* + E) \right)
    .\end{split}
\end{equation*}
By independence of $X$ and $E$, we have
\begin{equation*}
    \begin{split}
        \lim_{d_J \to \infty} \E \left[ \norm{\Sigma - \hat{\Sigma}} \cdot \norm{\wmn}^2 \right] &\geq \E \lim_{d_J \to \infty} \norm{\Sigma - \hat{\Sigma}} \cdot \left( E\tp  (X_S X_S\tp +\lambda_n I_n)^{-1} E \right)\\
        &= \sigma^2 \E \left[ \lim_{d_J \to \infty} \norm{\Sigma - \hat{\Sigma}} \cdot \tr \left(  (X_S X_S\tp +\lambda_n I_n)^{-1} \right) \right]\\
        &\geq \sigma^2 \E \left[ \lim_{d_J \to \infty} \norm{\Sigma - \hat{\Sigma}} \cdot \left(  \frac{n-d_S}{\lambda_n} \right) \right]\\
        &= \left(  \sigma^2 \frac{n-d_S}{\lambda_n} \right) \E \left[ \lim_{d_J \to \infty} \norm{\Sigma - \hat{\Sigma}} \right]
    .\end{split}
\end{equation*}
Next we want to interchange limit and expectation. Note that
\begin{equation*}
    \begin{split}
        \norm{\Sigma - \hat{\Sigma}}
        &\leq \norm{\Sigma} + \norm{\hat{\Sigma}} \\
        &= \norm{\Sigma} + \norm*{\frac{X_S\tp X_S + X_J X_J\tp}{n}}\\
        &\leq \norm{\Sigma} + \norm*{\frac{X_S\tp X_S}{n}} + \tr \left( \frac{X_J X_J\tp}{n} \right)\\
        &= \norm{\Sigma} + \norm*{\frac{X_S\tp X_S}{n}} + \frac{\lambda_n}{n} \tr \left( \frac{Z_J Z_J\tp}{d_J} \right)
    .\end{split}
\end{equation*}
The first two terms do not depend on $d_J$. It is easy to verify that
\[ \lim_{d_J \to \infty} \E \left[ \frac{\lambda_n}{n} \tr \left( \frac{Z_J Z_J\tp}{d_J} \right) \right] = \lambda_n =  \E \left[ \lim_{d_J \to \infty} \frac{\lambda_n}{n} \tr \left( \frac{Z_J Z_J\tp}{d_J} \right) \right] \]
as $\frac{Z_J Z_J\tp}{d_J} \asto I_n$.
Therefore, by the dominated convergence theorem
\begin{equation*}
    \lim_{d_J \to \infty} \E \left[ \norm{\Sigma - \hat{\Sigma}} \cdot \norm{\wmn}^2 \right]
    \geq \lim_{d_J \to \infty}  \left(  \sigma^2 \frac{n-d_S}{\lambda_n} \right) \E \norm{\Sigma - \hat{\Sigma}}
.\end{equation*}

\Citet{sample-covariance} show that, for Gaussian data,
\[
    \E \norm{\Sigma - \hat{\Sigma}}
    \ge C \max\left(
        \sqrt{ \frac{\Tr(\Sigma) \, \norm\Sigma}{n} },
        \frac{\Tr(\Sigma)}{n}
    \right)
,\]
where $C$ is a universal constant.
Thus, in our case
\[ \E \norm{\Sigma - \hat{\Sigma}} \ge C \sqrt{ \frac{d_S + \lambda_n}{n}} .\]
Since $\lambda_n = o(n)$, this implies 
\begin{equation*}
    \begin{split}
        \lim_{n\to \infty} \lim_{d_J \to \infty} \E \left[ \norm{\Sigma - \hat{\Sigma}} \cdot \norm{\wmn}^2 \right]
        &\ge \lim_{n\to \infty}  \left(  \sigma^2 \frac{n-d_S}{\lambda_n} \right) C \sqrt{ \frac{d_S + \lambda_n}{n}} = \infty
    .\end{split}
\end{equation*}

It is easy to see that the remaining terms in the lower bound of \cref{thm:ballgapbound} are negligible.
\end{proof}

\subsection{Uniform convergence on tighter sets\texorpdfstring{ (\cref{sec:two-sided})}{}} \label{sec:proofs:two-sided}

\zicostyle*
\begin{proof}
Fix any $\sampsdelta$ satisfying $\Pr(\samp \in \sampsdelta) \ge 1 - \delta$. For each $\samp = ((X_S, X_J), Y)$, we define $\sampb = ((X_S, -X_J), Y)$. Note that the marginal distribution of $\sampb$ is the same as $\samp$ because of the isotropic Gaussian distribution. Thus we also have $\Pr(\sampb \in \sampsdelta) \ge 1 - \delta$. By a simple union bound
\begin{equation*}
    \begin{split}
        1 - \Pr (\samp \in \sampsdelta \, \cap \, \sampb \in \sampsdelta) &= \Pr (\samp \not\in \sampsdelta \, \cup \, \sampb \not\in \sampsdelta)\\
        &\leq \Pr (\samp \not\in \sampsdelta) + \Pr (\sampb \not\in \sampsdelta) \leq 2 \delta
    .\end{split}
\end{equation*}
As $\delta < \frac{1}{2}$, we have $\Pr (\samp \in \sampsdelta \, \cap \, \sampb \in \sampsdelta) > 0$, so the set $\{\samp \in \sampsdelta: \sampb \in \sampsdelta \}$ must be nonempty.
Pick any $\samp = ( (X_S, X_J), Y )$ in this set; thus
$\hat w = \A\left( (X_S, X_J), Y \right) \in \W_{n,\delta}$
and $\tilde w = \A\left( (X_S, -X_J), Y \right) \in \W_{n,\delta}$.
As $\A$ outputs interpolators, we have that
\[
  X_S \hat w_S + X_J \hat w_J = Y = X_S \tilde w_S - X_J \tilde w_J
,\]
and \eqref{eq:a-flip-cond} implies that $\hat w_S = \tilde w_S$,
so then $X_J \hat w_J = - X_J \tilde w_J$.
Thus
\begin{align*}
       L_\samp(\tilde w)
  &  = \frac1n \norm{X \tilde w - Y}^2
     = \frac1n \norm{X_S \hat w_S - X_J \hat w_J - (X_S \hat w_S + X_J \hat w_J)}^2
     = \frac1n \norm{-2 X_J \hat w_J}^2
\\&\ge \frac4n \norm{(I_n - \Pi) X_J \hat w_J}^2
,\end{align*}
where $\Pi \in \R^{n \times n}$ is the orthogonal projection onto the range of $X_S$.
Now,
\begin{align*}
        (I_n - \Pi) X_J \hat w_J
  &   = (I_n - \Pi) (X_S \hat w_S + X_J \hat w_J)
\\&   = (I_n - \Pi) Y
\\&   = (I_n - \Pi) (X_S w_S^* + E)
\\&   = (I_n - \Pi) E
\\&\sim \N(0, \sigma^2 (I_n - \Pi))
\end{align*}
using $E \sim \N(0, \sigma^2 I_n)$.
As $n \to \infty$, because $X_S$ is almost surely rank $d_S$, $\Tr(I_n - \Pi)$ is almost surely $n - d_S$.
Thus we have
\[
  \frac{1}{n-d_S} \norm{(I_n - \Pi) X_J \hat w_J}^2
  \asto \sigma^2 
,\]
and so
\[
   L_\samp(\tilde w)
  \asge 4 \sigma^2 \frac{n - d_S}{n} \to 4 \sigma^2
.\]
The conclusion follows by the observation that
\[
\sup_{\samp \in \sampsdelta} \sup_{w \in \W_\delta} \abs{L_\D(w) - L_\samp(w)} \geq L_\samp(\tilde w) - L_\D(\tilde w)
.\qedhere\]
\end{proof}

\begin{restatable}{prop}{flipalg} \label{thm:flip-alg}
  Let $f_S : \R^{d_S} \to \R$ and $f_J : \R^{d_J} \to \R$ be convex functions,
  with $f_J$ symmetric, $f_J(-w) = f_J(w)$.
  Let $\A$ be an interpolation algorithm satisfying
  \[
    \A(X, y) = \argmin_{w \st X w = y} f_S(w_S) + f_J(w_J)
  .\]
  Then negating junk dimensions simply negates the corresponding dimensions of the predictor:
  \[
    \A\left( \left( X_S, - X_J \right), Y \right)
    = \begin{bmatrix} I_{d_S} & 0_{d_S \times d_J} \\ 0_{d_J \times d_S} & -I_{d_J} \end{bmatrix}
      \A\left( \left(X_S, X_J\right), Y \right)
  .\]
  (If the minimizer is not unique, the equation holds as an operation on sets.)
\end{restatable}
\begin{proof}
The KKT conditions for $\A(X, y)$,
which are both necessary and sufficient in this case, are
\begin{equation} \label{eq:kkt-orig}
  X w = X_S w_S + X_J w_J = Y
  ,\qquad
  0 \in \partial f_S(w_S) + \nu_S\tp X_S
  ,\qquad
  0 \in \partial f_J(w_J) + \nu_J\tp X_J
,\end{equation}
where $\delta$ denotes the subdifferential,
and the dual variables $\nu_S \in \R^{d_S}$ and $\nu_J \in \R^{d_J}$ are otherwise unconstrained.
Also note that because $f_J$ is symmetric,
if $g \in \partial f_J$ then for any $t$, there is some $g' \in \partial f_J$ such that $g'(-t) = -g(t)$.

Let $(\hat w, \nu_S, \nu_J)$ be some solution to \eqref{eq:kkt-orig},
and define $\tilde w = \left( \hat w_S, -\hat w_J \right)$,
$\tilde X = \left( X_S, - X_J \right)$.
Then
\begin{gather*}
  \left( X_S, - X_J \right) \tilde w
  = X_S \tilde w_S - X_J \tilde w_J
  = X_S \hat w_S + X_J \hat w_J
  = Y
  ,\\
  \partial f_S(\tilde w_S) + \nu_S\tp \tilde X_S
  = \partial f_S(\hat w_S) + \nu_S\tp X_S
  \ni 0
  ,\\\text{and}\quad
  \partial f_J(\tilde w_J) + \nu_J\tp \tilde X_J
  = \partial f_J(-\hat w_J) + \nu_J\tp (-X_J)
  \ni 0
  \quad\text{because}\quad
  0 \in \partial f_J(\hat w_J) + \nu_J\tp X_J
.\end{gather*}
Thus $(\tilde w, \nu_S, \nu_J)$ satisfies the KKT conditions for $\A( \tilde X, Y )$.
When the minimizer is not unique, the same argument works in reverse, showing that solution sets are related in the same way.
\end{proof}

\section{Proofs for Section \ref{sec:interp-convergence}}
\subsection{Consistency of the minimal risk interpolator\texorpdfstring{ (\cref{prop:wmr-consistent})}{}} \label{sec:proofs:wmr-consistent}

\wmrconsistent*

\begin{proof}
Recall that
\[
    \wmr = w^* + \Sigma^{-1} X\tp (X \Sigma^{-1} X\tp)^{-1} E
.\]
From this, we can compute
\begin{equation*}
    \begin{split}
        L_{\D}(\wmr) - L_{\D}(w^*) &= (\wmr - w^*)\tp \Sigma (\wmr - w^*) \\
        &= (\wmr - w^*)\tp X\tp (X \Sigma^{-1} X\tp)^{-1} E\\
        &= (X\wmr - Xw^*)\tp (X \Sigma^{-1} X\tp)^{-1} E\\
        &= (Y - Xw^*)\tp (X \Sigma^{-1} X\tp)^{-1} E\\
        &= E \tp (ZZ\tp)^{-1} E \\
        &= \inner{ (ZZ\tp)^{-1} }{ E E\tp }
    .\end{split}
\end{equation*}
By independence of $Z$ and $E$, we get
\[
    \E [L_{\D}(\wmr) - L_{\D}(w^*)]
    = \sigma^2 \E \tr \left[ \left( ZZ\tp \right)^{-1} \right]
.\]

Note that $\left( ZZ\tp \right)^{-1}$ follows an inverse-Wishart distribution whose expectation is $\frac{I_n}{p-n-1}$. Therefore, we obtain
\begin{equation*}
    \begin{split}
        \E [L_{\D}(\wmr) ] &= \sigma^2 + \sigma^2 \tr \left( \frac{I_n}{p-n-1} \right) \\
        &= \sigma^2 \left( 1 + \frac{n}{p-n-1}\right) = \left( \frac{p-1}{p-n-1} \right) \cdot L_{\D} (w^*)
    .\qedhere\end{split}
\end{equation*}
\end{proof}

\subsection{Uniform consistency of low norm interpolators\texorpdfstring{ (\cref{sec:non-min-norm})}{}} \label{sec:proofs:non-min-norm}

\subsubsection{General results}
Our key lemma is as follows:

\begin{lemma} \label{key-lemma}
Let $\hat{w}$ be any predictor that interpolates the data, with $\norm{\hat w} \le B$, and $F \in \R^{p \times (p-n)}$  be the matrix whose columns form an orthonormal basis of the kernel of $X$.
In other words, if $X\hat{w} = Y$, $XF = 0_{n \times (p-n)} $ and $F\tp F = I_{p-n}$, then \eqref{eq:gen-gap}, the worst-case generalization gap for interpolators up to norm $B$, is equal to
\[
    L_{\D}(\hat{w}) + \inf_{\lambda > \norm{ F\tp \Sigma F}} \norm{F\tp [ \lambda \hat{w} - \Sigma(\hat{w} - w^*)] }_{(\lambda I_{p-n} - F\tp \Sigma F)^{-1}} +  \lambda (B^2 - \norm{\hat{w}}^2)
.\]
\end{lemma}
\begin{proof}
Observe that $\{w \in \R^p: L_S(w) = 0 \} = \{\hat{w} + Fu: u \in \R^{p-n} \}$.
Then
\begin{equation*}
    \begin{split}
        &\sup_{\substack{\norm{w} \leq B \\ L_S(w) = 0}} L_{\D}(w) - L_S(w) \\
        &= L_\D(w^*) + \sup_{\substack{\norm{w} \leq B \\ L_S(w) = 0}} L_{\D}(w) - L_{\D}(w^*) \\
        &= L_\D(w^*) + \sup_{\substack{\norm{\hat{w} + Fu}^2 \leq B^2}} (\hat{w} + Fu - w^*)\tp \Sigma (\hat{w} + Fu - w^*)\\
        &= L_\D(w^*) + \sup_{\substack{\norm{u}^2 + 2 \langle u, F\tp \hat{w} \rangle + \norm{\hat{w}}^2 \leq B^2}} u\tp (F\tp \Sigma F) u + 2 \langle u, F\tp\Sigma (\hat{w} - w^*)\rangle + (\hat{w} - w^*)\tp \Sigma (\hat{w} - w^*) \\
        &= L_\D(\hat w) + \sup_{\substack{\norm{u}^2 + 2 \langle u, F\tp \hat{w} \rangle + \norm{\hat{w}}^2 \leq B^2}} u\tp (F\tp \Sigma F) u + 2 \langle u, F\tp\Sigma (\hat{w} - w^*)\rangle  \\
        &= L_\D(\hat w) -\inf_{\substack{\norm{u}^2 + 2 \langle u, F\tp \hat{w} \rangle + \norm{\hat{w}}^2 \leq B^2}} u\tp (-F\tp \Sigma F) u - 2 \langle u, F\tp\Sigma (\hat{w} - w^*)\rangle
.\end{split}
\end{equation*}
Although the second term involves a concave minimization problem, it is a quadratic optimization problem with a single quadratic inequality constraint.
This is a classical example where strong duality holds even though the objective is not convex \citep[Appendix B]{convex-optimization}.
In order to derive the dual, we write down the Lagrangian:
\begin{equation*}
    \begin{split}
        L(u, \lambda) &= u\tp (-F\tp \Sigma F) u - 2 \langle u, F\tp\Sigma (\hat{w} - w^*)\rangle + \lambda (\norm{u}^2 + 2 \langle u, F\tp \hat{w} \rangle + \norm{\hat{w}}^2 - B^2) \\
        &= u\tp (\lambda I_{p-n} - F\tp \Sigma F) u + 2 \langle u, F\tp(\lambda \hat{w} - \Sigma (\hat{w} - w^*)) \rangle - \lambda (B^2 - \norm{\hat{w}}^2)
    ;\end{split}
\end{equation*}
strong duality tells us that the infimum is equal to
$\sup_{\lambda \ge 0} \inf_u L(u, \lambda)$.
For $\lambda < \norm{F\tp \Sigma F}$,
$\lambda I_{p-n} - F\tp \Sigma F$ has strictly negative eigenvalues,
and so then $\inf_u L(u, \lambda) = -\infty$.
If instead $\lambda > \norm{F\tp \Sigma F}$,
$\lambda I_{p-n} - F\tp \Sigma F$ is strictly positive definite,
and setting the $u$ derivative to zero yields that $\inf_u L(\lambda, u)$ is
\begin{equation} \label{eq:the-dual}
    -\left[F\tp( \lambda \hat{w} - \Sigma(\hat{w} - w^*))\right]\tp (\lambda I_{p-n} - F\tp \Sigma F)^{-1} \left[F\tp( \lambda \hat{w} - \Sigma(\hat{w} - w^*))\right] - \lambda (B^2 - \norm{\hat{w}}^2)
.\end{equation}
If instead $\lambda = \norm{F\tp \Sigma F}$,
we again have $\inf_u L(u, \lambda) = -\infty$ unless
$
    F\tp ( \lambda \hat w - F\tp \Sigma (\hat w - w^*) ) = 0
$
so that the linear term is identically zero;
in this case,
the quadratic term is minimized by $u = 0$, and
$\inf_u L(u, \lambda) = \lambda (B^2 - \norm{\hat w}^2)$ agrees with \eqref{eq:the-dual},
so this case is covered by the strict case as well.
Thus the dual problem is to maximize \eqref{eq:the-dual} over $\lambda > \norm{F\tp \Sigma F}$.
The desired result follows by passing the minus sign into the sup of the dual problem.
\end{proof}

We will now prove \cref{thm:general-consistency}.

\generalconsistency*
\begin{proof}
For case \ref{thm-part:general-consistency:wmr},
observe that
\[
  F\tp \Sigma (\wmr - w^*)
  = F\tp X\tp (X \Sigma^{-1} X\tp)^{-1} E
  = (XF)\tp (X \Sigma^{-1} X\tp)^{-1} E
  = 0
.\]
Thus picking $\hat{w} = \wmr$ and $B = \norm\wmr$ in \cref{key-lemma} gives that
\begin{equation}
    \sup_{\norm w \le \norm \wmr,\, L_\samp(w) = 0} L_\D(w)
    =
    L_{\D}(\wmr) + \inf_{\lambda > \norm{ F\tp \Sigma F}} \norm{\lambda F\tp \wmr }_{(\lambda I_{p-n} - F\tp \Sigma F)^{-1}}
\label{eq:gen:wmr}
.\end{equation}
Since we have
\[
  \frac{1}{\lambda} I_{p-n} \preceq (\lambda I_{p-n} - F\tp \Sigma F)^{-1} 
,\]
we know that
$\sup_{\norm w \le \norm \wmr,\, L_\samp(w) = 0} L_\D(w)$
is lower bounded by 
\[
  L_{\D}(\wmr) + \inf_{\lambda > \norm{F\tp \Sigma F}} \frac{1}{\lambda} \norm{ \lambda F\tp \wmr}^2 = L_{\D}(\wmr) + \norm{F\tp \Sigma F} \cdot \norm{F\tp \wmr}^2
.\]
In order to compute $\norm{F\tp \wmr}^2$, we notice that $FF\tp$ is the orthogonal projection onto the kernel of $X$.
Using the fact that $\text{im}(X\tp) = \text{ker}(X)^{\bot}$, we get $I - FF\tp$ is the orthogonal projection onto the image of $X\tp$.
Thus,
\[
  X(I - FF\tp)\wmr = X\wmr = Y
,\]
and left-multiplying both sides by $X\tp (X X\tp)^{-1}$
gives that
\[
  \wmn
  = X\tp (X X\tp)^{-1} X (I - FF\tp) \wmr
  = (I - FF\tp) \wmr
,\]
and so
\begin{equation*}
    \begin{split}
        \norm{F\tp \wmr}^2 & = \wmr\tp F F\tp \wmr\\
        & = \wmr\tp F(F\tp F)F\tp \wmr\\
        & = \norm{FF\tp \wmr}^2 \\
        &= \norm{\wmr}^2 - \norm{(I-FF\tp)\wmr}^2\\
        &= \norm{\wmr}^2 - \norm{\wmn}^2\\
    \end{split}
\end{equation*}
which establishes the lower bound with a constant of $1$.

Similarly,
we can
use $(\lambda I_{p-n} - F\tp \Sigma F)^{-1} \preceq \frac{1}{\lambda - \norm{F\tp \Sigma F}} I_{p-n}$
to upper bound \eqref{eq:gen:wmr} as
\begin{equation*}
    \begin{split}
        &L_{\D}(\wmr) + \inf_{\lambda > \norm{F\tp \Sigma F}} \frac{1}{\lambda - \norm{F\tp \Sigma F}} \norm{ \lambda F\tp \wmr}^2 \\
        = \, &L_{\D}(\wmr) + \inf_{\lambda > 0} \, \, \frac{(\lambda + \norm{F\tp \Sigma F})^2}{\lambda} (\norm{\wmr}^2 - \norm{\wmn}^2)\\
        = \, &L_{\D}(\wmr) + \inf_{\lambda > 0} \, \, \left( \lambda + 2 \norm{F\tp \Sigma F} + \frac{\norm{F\tp \Sigma F}^2}{\lambda} \right) (\norm{\wmr}^2 - \norm{\wmn}^2)\\
        = \, &L_{\D}(\wmr) + 4 \norm{F\tp \Sigma F} \cdot (\norm{\wmr}^2 - \norm{\wmn}^2)
    .\end{split}
\end{equation*}
This gives the desired upper bound with a constant of $4$.
It follows immediately that \eqref{eq:gen:wmr} converges to $L_{\D}(w^*)$ if and only if
\[ \E \norm{F\tp \Sigma F} \cdot (\norm{\wmr}^2 - \norm{\wmn}^2) \to 0 .\]

Turning to part \ref{thm-part:general-consistency:ball},
observe that
\[
    F\tp \wmn = F\tp X\tp (XX\tp)^{-1} Y = (XF)\tp (XX\tp)^{-1} Y = 0
,\]
so that \cref{key-lemma} with $\hat w = \wmn$ gives
\[
    \sup_{\norm w \le B_n\, L_\samp(w) = 0} L_\D(w)
    =
    L_\D(\wmn) + \inf_{\lambda > \norm{ F\tp \Sigma F}} \norm{F\tp \Sigma(\hat{w} - w^*) }_{(\lambda I_{p-n} - F\tp \Sigma F)^{-1}} +  \lambda (B_n^2 - \norm{\wmn}^2)
.\]
Moreover, it is clear that
\[
    0_{p-n} \prec (\lambda I_{p-n} - F\tp \Sigma F)^{-1} \prec \frac{1}{\lambda - \norm{F\tp \Sigma F}} I_{p-n}
.\]
Therefore, $\sup_{\norm w \le B_n, \, L_\samp(w) = 0} L_\D(w)$ is lower bounded by,
recalling that $\norm{F\tp \Sigma F} = \kappa_X(\Sigma)$,
\begin{equation}
    L_{\D}(\wmn) + \inf_{\lambda > \norm{F\tp \Sigma F}} \lambda (B_n^2 - \norm{\wmn})
    = L_{\D}(\wmn) + \kappa_X(\Sigma) \cdot \left[ B_n^2 - \norm{\wmn}^2 \right]
\label{eq:gen:ball:lower}
,\end{equation}
and we have shown that $R_n \ge 0$ in the result.
On the other hand, $\sup_{\norm w \le B_n, \, L_\samp(w) = 0} L_\D(w)$ is upper bounded by
\begin{align*}
    &L_{\D}(\wmn) + \inf_{\lambda > \norm{F\tp \Sigma F}} \frac{1}{\lambda- \norm{F\tp \Sigma F}} \norm{ F\tp \Sigma(\wmn - w^*)}^2 + \lambda \Big[ B_n^2 - \norm{\wmn}^2 \Big] \\
    = \, & L_{\D}(\wmn) + \inf_{\lambda > 0} \, \, \frac{1}{\lambda} \norm{ F\tp \Sigma(\wmn - w^*)}^2 + (\lambda + \kappa_X(\Sigma)) \Big[ B_n^2 - \norm{\wmn}^2 \Big]\\
    = \, & L_{\D}(\wmn) +  \kappa_X(\Sigma) \cdot \Big[ B_n^2 - \norm{\wmn}^2 \Big] + \inf_{\lambda > 0} \, \, \frac{1}{\lambda} \norm{ F\tp \Sigma(\wmn - w^*)}^2 + \lambda \Big[ B_n^2 - \norm{\wmn}^2 \Big] \\
    = \, & L_{\D}(\wmn) +  \kappa_X(\Sigma) \cdot \Big[ B_n^2 - \norm{\wmn}^2 \Big] + 2 \sqrt{\norm{F\tp \Sigma(\wmn - w^*)}^2 \cdot \Big[ B_n^2 - \norm{\wmn}^2 \Big] }
    \tagthis\label{eq:gen:ball:upper}
.\end{align*}
We can upper bound
\begin{equation*}
    \begin{split}
        \norm{F\tp \Sigma(\wmn - w^*)}^2
        &= (\wmn - w^*)\tp \Sigma F F\tp \Sigma(\wmn - w^*)\\
        &= [\Sigma^{1/2}(\wmn - w^*)]\tp (\Sigma^{1/2} F F\tp \Sigma^{1/2}) [\Sigma^{1/2}(\wmn - w^*)]\\
        &\leq \norm{\Sigma^{1/2} F F\tp \Sigma^{1/2}} \cdot \norm{\Sigma^{1/2}(\wmn - w^*)}^2\\
        &= \norm{F\tp \Sigma F} \cdot [L_{\D}(\wmn) - L_{\D}(w^*) ]\\
    \end{split}
,\end{equation*}
using the fact that $\norm{AA^T} = \norm{A^T A}$ with $A = F^T \Sigma^{1/2}$. 
Plugging into the third term of \eqref{eq:gen:ball:upper} yields our desired upper bound on $R_n$,

To show the statement about expectations when $\E L_\D(\wmn) - L_\D(w^*) \to 0$,
note for one direction that \eqref{eq:gen:ball:lower} gives
\[
        \liminf_{n \to \infty} \E \left[ \sup_{\substack{\norm w \leq B_n \\ L_\samp(w) = 0}} L_\D(w) - L_\samp(w) \right]
        \geq L_\D(w^*) + \lim_{n \to \infty} \E \kappa_X(\Sigma) \cdot \Big[ B_n^2 - \norm{\wmn}^2 \Big]
.\]
For the other direction, we have
\begin{multline*}
        R_n \le 
        2 \sqrt{ \norm{F\tp \Sigma F} \cdot [L_{\D}(\wmn) - L_{\D}(w^*) ] \Big[ B_n^2 - \norm{\wmn}^2 \Big] } \\
        \leq \epsilon \norm{F\tp \Sigma F} \cdot \Big[ B_n^2 - \norm{\wmn}^2 \Big]  + \frac{1}{\epsilon} [L_{\D}(\wmn) - L_{\D}(w^*) ] \\
\end{multline*}
for any $\epsilon > 0$. This implies 
\[
        \limsup_{n \to \infty} \E \left[ \sup_{\substack{\norm w \leq B_n \\ L_\samp(w) = 0}} L_\D(w) - L_\samp(w) \right]
        \leq L_\D(w^*) + (1 + \epsilon) \E \left( \lim_{n \to \infty} \kappa_X(\Sigma) \cdot \Big[ B_n^2 - \norm{\wmn}^2 \Big] \right)
,\]
showing the desired result.
\end{proof}

\subsubsection{Special case of Setting \ref{setting:junk-feats}}
In \cref{setting:junk-feats},
we are able to compute $\kappa_X(\Sigma)$.
\begin{prop} \label{thm:junk-re}
With probability 1, it holds in \cref{setting:junk-feats} that
\[
    \lim_{d_J \to \infty} \kappa_X(\Sigma)
    = \frac{\lambda_n}{n} \norm*{ \left[
        \frac{X_S\tp X_S}{n} + \frac{\lambda_n}{n} I_{d_S}
    \right]^{-1}  }
.\]
\end{prop}

\begin{proof}
Recall that 
\[ 
    \kappa_X(\Sigma) = \norm{F\tp \Sigma F} = \norm{\Sigma^{1/2} FF\tp \Sigma^{1/2}} = \norm{\Sigma^{1/2} (I - X\tp(XX\tp)^{-1} X) \Sigma^{1/2}} 
.\]
It is a routine calculation to show that
\[
  \Sigma^{1/2} FF\tp \Sigma^{1/2}
  = \begin{bmatrix}
     I_{d_S} - X_S\tp (X_S X_S\tp + X_J X_J\tp)^{-1} X_S
     &
     -\sqrt{\frac{\lambda_n}{d_J}} X_S\tp (X_S X_S\tp + X_J X_J\tp)^{-1} X_J
     \\
     -\sqrt{\frac{\lambda_n}{d_J}} X_J\tp (X_S X_S\tp + X_J X_J\tp)^{-1} X_S
     &
     \frac{\lambda_n}{d_J} \left[ I_{d_J} - X_J\tp (X_S X_S\tp + X_J X_J\tp)^{-1} X_J \right]
  \end{bmatrix}
 .\]
Intuitively, since only the upper-left block does not vanish as $d_J \to \infty$, we should expect 
\[
  \lim_{d_J \to \infty} \, \kappa_X(\Sigma) = \norm{I_{d_S} - X_S\tp (X_S X_S\tp + \lambda_n I_n)^{-1} X_S }
.\] 
However, as the dimensions of $\Sigma^{1/2} FF\tp \Sigma^{1/2}$ also increase with $d_J$, the analysis of $\kappa_X(\Sigma)$ requires more care. 

It is clear that $\kappa_X(\Sigma) \geq \norm{ I_{d_S} - X_S\tp (X_S X_S\tp + X_J X_J\tp)^{-1} X_S }$, and so
\[
  \liminf_{d_J \to \infty} \, \kappa_X(\Sigma) \geq \norm{ I_{d_S} - X_S\tp (X_S X_S\tp + \lambda_n I_n)^{-1} X_S }
.\]
To upper bound the limit, fix any $v = (v_1, v_2)$ such that $v_1 \in \R^{d_S}$, $v_2 \in \R^{d_J}$ and $\norm{v} = 1$.
We can write
\begin{equation}
\begin{multlined}
  v\tp \Sigma^{1/2} FF\tp \Sigma^{1/2} v
  = v_1\tp (I_{d_S} - X_S\tp (X_S X_S\tp + X_J X_J\tp)^{-1} X_S ) v_1
  \\
  + \frac{\lambda_n}{d_J}  v_2\tp \left[ I_{d_J} - X_J\tp (X_S X_S\tp + X_J X_J\tp)^{-1} X_J \right] v_2
  \\
  - 2 \sqrt{\frac{\lambda_n}{d_J}}  v_1\tp  X_S\tp (X_S X_S\tp + X_J X_J\tp)^{-1} X_J v_2
.\end{multlined}
\label{eq:junk:kappa-split}
\end{equation}
The first term is upper bounded by 
\[
  \norm{ I_{d_S} - X_S\tp (X_S X_S\tp + X_J X_J\tp)^{-1} X_S } \cdot \norm{v_1} \leq \norm{ I_{d_S} - X_S\tp (X_S X_S\tp + X_J X_J\tp)^{-1} X_S }
,\]
and the second term is upper bounded by $\lambda_n/d_J$, because
\[
  v_2\tp v_2 \leq 1
  \qquad \text{ and } \qquad
  v_2\tp X_J\tp (X_S X_S\tp + X_J X_J\tp)^{-1} X_J v_2 \geq 0
.\]

For any $\epsilon > 0$, we have
\begin{equation*}
    \begin{split}
        &- 2 \sqrt{\frac{\lambda_n}{d_J}}  v_1\tp  X_S\tp (X_S X_S\tp + X_J X_J\tp)^{-1} X_J v_2 \\
        \leq & \, 2 \norm{ v_1\tp  X_S\tp (X_S X_S\tp + X_J X_J\tp)^{-1/2}} \cdot \norm*{ \sqrt{\frac{\lambda_n}{d_J}} (X_S X_S\tp + X_J X_J\tp)^{-1/2} X_J v_2 }\\
        \leq & \, \epsilon \norm{ v_1\tp  X_S\tp (X_S X_S\tp + X_J X_J\tp)^{-1/2}}^2 + \frac{1}{\epsilon} \norm*{ \sqrt{\frac{\lambda_n}{d_J}} (X_S X_S\tp + X_J X_J\tp)^{-1/2} X_J v_2 }^2\\
        \leq & \, \epsilon \norm{X_S\tp (X_S X_S\tp + X_J X_J\tp)^{-1} X_S} + \frac{\lambda_n}{\epsilon \, d_J} \norm{X_J\tp (X_S X_S\tp + X_J X_J\tp)^{-1} X_J}\\
        = & \, \epsilon \norm{X_S\tp (X_S X_S\tp + X_J X_J\tp)^{-1} X_S} + \frac{\lambda_n}{\epsilon \, d_J} \norm{ (X_S X_S\tp + X_J X_J\tp)^{-1/2} X_J X_J\tp (X_S X_S\tp + X_J X_J\tp)^{-1/2}}
    .\end{split}
\end{equation*}

Taking a supremum over $v$ in \eqref{eq:junk:kappa-split}, we get
\begin{equation*}
    \begin{split}
        \kappa_X(\Sigma) &\leq \norm{ I_{d_S} - X_S\tp (X_S X_S\tp + X_J X_J\tp)^{-1} X_S } + \epsilon \norm{ X_S\tp (X_S X_S\tp + X_J X_J\tp)^{-1} X_S}\\
        & \qquad + \frac{\lambda_n}{d_J} \left[ 1 + \frac{1}{\epsilon} \norm{ (X_S X_S\tp + X_J X_J\tp)^{-1/2} X_J X_J\tp (X_S X_S\tp + X_J X_J\tp)^{-1/2}} \right]
    .\end{split}
\end{equation*}
Note that 
\[
\begin{multlined}
  \lim_{d_J \to \infty} \norm{ (X_S X_S\tp + X_J X_J\tp)^{-1/2} X_J X_J\tp (X_S X_S\tp + X_J X_J\tp)^{-1/2}}
  \\
  = \lambda_n \norm{(X_S X_S\tp + \lambda_n I_n)^{-1} } < \infty
,\end{multlined}
\]
so for any $\epsilon > 0$, 
\[
  \limsup_{d_J \to \infty } \, \kappa_X(\Sigma)
  \leq \norm{ I_{d_S} - X_S\tp (X_S X_S\tp + \lambda_n I_n)^{-1} X_S } + \epsilon \norm{X_S\tp (X_S X_S\tp + \lambda_n I_n)^{-1} X_S}
.\]
Sending $\epsilon \to 0$ matches the $\liminf$ and $\limsup$.
Finally, because 
\[
  (X_S X_S\tp + \lambda_n I_n)^{-1} X_S = X_S (X_S\tp X_S + \lambda_n I_{d_S})^{-1}
,\]
we have
\begin{equation*}
    \begin{split}
        I_{d_S} - X_S\tp (X_S X_S\tp + \lambda_n I_n)^{-1} X_S &= I_{d_S} - X_S\tp X_S (X_S\tp X_S + \lambda_n I_{d_S})^{-1} \\
        &= \lambda_n (X_S\tp X_S + \lambda_n I_{d_S})^{-1} \\
        &= \frac{\lambda_n}{n} \left[
        \frac{X_S\tp X_S}{n} + \frac{\lambda_n}{n} I_{d_S} \right]^{-1}  \\
    \end{split}
\end{equation*}
and the proof is concluded.
\end{proof}

\begin{prop} \label{thm:product}
In \cref{setting:junk-feats}, it holds that
\begin{gather*}
    \lim_{n \to \infty} \lim_{d_J \to \infty} \E \kappa_X(\Sigma) \cdot \norm{\wmn}^2  = L_{\D}(w^*)
    ,\\
    \lim_{n \to \infty} \lim_{d_J \to \infty} \E \kappa_X(\Sigma) \cdot \left[ \norm{\wmr}^2 - \norm{\wmn}^2\right] = 0
.\end{gather*}
\end{prop}

\begin{proof}
Notice that $\kappa_X(\Sigma) \cdot \norm{\wmn}^2$ can be dominated by $\norm{\Sigma} \cdot \norm{\wmr}^2$ and \cref{thm:mn-mr-norms} showed that $\norm{\wmr}^2$ is integrable, so by the dominated convergence theorem,
\[
  \lim_{d_J \to \infty} \E \, \kappa_X(\Sigma) \cdot \norm{\wmn}^2 =  \E \lim_{d_J \to \infty} \kappa_X(\Sigma) \cdot \norm{\wmn}^2
.\] 
Similarly, $\, \lim_{d_J \to \infty} \kappa_X(\Sigma) \cdot \norm{\wmn}^2 \,$  can be dominated by 
\[ \lim_{d_J \to \infty} \kappa_X(\Sigma) \cdot \norm{\wmr}^2 \aseq \frac{\lambda_n}{n} \norm*{ \left[
    \frac{X_S\tp X_S}{n} + \frac{\lambda_n}{n} I_{d_S} \right]^{-1} } \cdot \left( \norm{w^*}^2 + \frac{\norm{E}^2}{\lambda_n} \right) \]
according to \cref{thm:junk-re,thm:mn-mr-norms}. 

As
$\norm*{ \left[ \frac{X_S\tp X_S}{n} + \frac{\lambda_n}{n} I_{d_S} \right]^{-1} } \asto 1$
and
$\frac{\norm{E}^2}{n} \asto \sigma^2$,
we have
\[ \lim_{n \to \infty} \lim_{d_J \to \infty} \kappa_X(\Sigma) \cdot \norm{\wmr}^2 \aseq \sigma^2 .\] 
Moreover, by independence of $X_S$ and $E$
\[ \E \lim_{d_J \to \infty} \kappa_X(\Sigma) \cdot \norm{\wmr}^2 = \left( \frac{\lambda_n \norm{w^*}^2}{n} + \sigma^2 \right) \cdot 
\E \norm*{ \left[ \frac{X_S\tp X_S}{n} + \frac{\lambda_n}{n} I_{d_S} \right]^{-1} }  .\] 
Again,
$\norm*{ \left[ \frac{X_S\tp X_S}{n} + \frac{\lambda_n}{n} I_{d_S} \right]^{-1} }$
can be dominated by $\tr \left ( \left( \frac{X_S\tp X_S}{n} \right)^{-1} \right)$,
so that
\[
  \lim_{n \to \infty} \E \lim_{d_J \to \infty} \kappa_X(\Sigma) \cdot \norm{\wmr}^2
  = \sigma^2
  = \E \lim_{n \to \infty} \lim_{d_J \to \infty} \kappa_X(\Sigma) \cdot \norm{\wmr}^2
.\]
It is also straightforward to check that
\[
  \lim_{n \to \infty} \E \frac{\lambda_n}{n} \left( \lim_{d_J \to \infty}  \norm{\wmr}^2 \right)
  = \sigma^2
  = \E \lim_{n \to \infty} \frac{\lambda_n}{n} \cdot \left( \lim_{d_J \to \infty} \norm{\wmr}^2 \right)
.\] 
Another application of DCT shows that
\begin{equation*}
    \begin{split}
        \lim_{n \to \infty} \lim_{d_J \to \infty} \E \, \kappa_X(\Sigma) \cdot \norm{\wmn}^2 &= \lim_{n \to \infty}  \E \lim_{d_J \to \infty} \kappa_X(\Sigma) \cdot \norm{\wmn}^2\\
        &=  \E \lim_{n \to \infty}  \lim_{d_J \to \infty} \kappa_X(\Sigma) \cdot \norm{\wmn}^2\\
        &= \E \lim_{n \to \infty} \frac{\lambda_n}{n} \norm*{ \left[
        \frac{X_S\tp X_S}{n} + \frac{\lambda_n}{n} I_{d_S}
        \right]^{-1}  } \cdot \left( \lim_{d_J \to \infty} \norm{\wmn}^2 \right)\\
        &= \E \lim_{n \to \infty} \frac{\lambda_n}{n} \cdot \left( \lim_{d_J \to \infty} \norm{\wmn}^2 \right)
    .\end{split}
\end{equation*}
Using the fact that
\[ \frac{\lambda_n}{n} \cdot \left( \lim_{d_J \to \infty} \norm{\wmn}^2 \right) \leq \frac{\lambda_n}{n} \cdot \left( \lim_{d_J \to \infty} \norm{\wmr}^2 \right)\]
and $\norm{\wmn}^2 \leq \norm{\wmr}^2$, two final applications of DCT give
\begin{equation*}
    \begin{split}
        \lim_{n \to \infty} \lim_{d_J \to \infty} \E \, \kappa_X(\Sigma) \cdot \norm{\wmn}^2 &= \lim_{n \to \infty} \frac{\lambda_n}{n} \, \left( \E  \lim_{d_J \to \infty} \norm{\wmn}^2 \right) \\
        &= \lim_{n \to \infty} \frac{\lambda_n}{n} \, \left( \lim_{d_J \to \infty}  \E \norm{\wmn}^2 \right) \\
        &=  \lim_{n \to \infty} \frac{\lambda_n}{n} \, \left[ \norm{w^*}^2 + \sigma^2 \frac{n-d_S}{\lambda_n} + \beta_n \left( \frac{\sigma^2 d_S - \lambda_n \norm{w_S^*}^2 }{n} \right) \right] \\
        &= \sigma^2
    .\end{split}
\end{equation*}
by \cref{thm:mn-mr-norms}. Consequently, we have established
\[
  \lim_{n \to \infty} \lim_{d_J \to \infty} \E\left[ \kappa_X(\Sigma) \cdot \left( \norm{\wmr}^2 - \norm{\wmn}^2\right) \right]
  = 0
.\qedhere\]
\end{proof}

We are finally ready to prove \cref{thm:junk-ball,thm:junk-wmr}.

\vspace{1ex}
\junkwmr*
\begin{proof}
Recall in the proof of \cref{thm:general-consistency},
it is shown that
\[
  \sup_{\norm{w} \le \norm\wmr, \, L_\samp(w) = 0} L_\D(w)
  \le
  L_{\D}(\wmr) + 4 \, \kappa_X(\Sigma) \cdot \left[ \norm{\wmr}^2 - \norm{\wmn}^2\right]
.\]
\cref{prop:wmr-consistent} implies that
\[
  \lim_{d_J \to \infty} \E  \, L_{\D}(\wmr) = L_{\D}(w^*)
.\]
Combined with \cref{thm:product}, we have shown
\[ 
        \lim_{n \to \infty} \lim_{d_J \to \infty} \E\left[
            \sup_{\norm{w} \le \norm\wmr, \, L_\samp(w) = 0} L_\D(w) - L_\samp(w) 
        \right]
        \leq L_{\D}(w^*)
.\]
On the other hand, we have the trivial lower bound
\[ 
        \lim_{n \to \infty} \lim_{d_J \to \infty} \E\left[
            \sup_{\substack{\norm{w} \le \norm\wmr\\L_\samp(w) = 0}} L_\D(w) - L_\samp(w) 
        \right]
        \geq \lim_{n \to \infty} \lim_{d_J \to \infty} \E L_{\D}(\wmr)
        =  L_{\D}(w^*)
. \qedhere \]
\end{proof}

\vspace{3ex}
\junkball*
\begin{proof}
In the proof of \cref{thm:general-consistency}, it is shown for every $\epsilon \geq 0$ that
\[
\sup_{\substack{\norm w \leq B_n \\ L_\samp(w) = 0}} L_\D(w) - L_\samp(w)
        \leq L_\D(\wmn) + (1 + \epsilon)  \kappa_X(\Sigma) \cdot \Big[ B_n^2 - \norm{\wmn}^2 \Big] + \frac{1}{\epsilon} [L_{\D}(\wmn) - L_{\D}(w^*) ]
.\]
\Cref{thm:junk-wmr} implies that $\lim_{n \to \infty} \lim_{d_J \to \infty} \E L_\D(\wmn) = L_\D(w^*)$.
Thus, plugging in $B_n = \alpha_n \norm{\wmn}$ and taking expectations and limits on both sides gives
\[
  \lim_{n \to \infty} \lim_{d_J \to \infty} \E
  \left[ \sup_{\substack{\norm w \leq \alpha_n \norm\wmn \\ L_\samp(w) = 0}} L_\D(w) \right]
  \leq L_\D(w^*) + (1 + \epsilon) \lim_{n\to\infty} \lim_{d_J \to \infty} \E (\alpha_n^2 - 1) \kappa_X(\Sigma) \norm{\wmn}^2
;\]
further applying \cref{thm:product} yields
\[
  \lim_{n \to \infty} \lim_{d_J \to \infty} \E
  \left[ \sup_{\substack{\norm w \leq \alpha_n \norm\wmn \\ L_\samp(w) = 0}} L_\D(w) \right]
    \leq L_\D(w^*) + (1+\epsilon) (\alpha^2 - 1) L_\D(w^*)
.\]
Sending $\epsilon \to 0$ yields the upper bound $\alpha^2 L_\D(w^*)$.

To get the lower bound, in the proof of \cref{thm:general-consistency} it is also shown
\[
  \sup_{\substack{\norm w \leq B_n \\ L_\samp(w) = 0}} L_\D(w) - L_\samp(w)
  \ge L_\D(\wmn) + \kappa_X(\Sigma) \cdot \Big[ B_n^2 - \norm{\wmn}^2 \Big]
.\]
By \cref{thm:product}, letting $B_n = \alpha_n \norm\wmn$ we obtain
\[ 
  \lim_{n \to \infty} \lim_{d_J \to \infty}
  \E\left[ \sup_{\norm{w} \le \alpha_n \norm\wmn, \, L_\samp(w) = 0} L_\D(w) \right]
  \geq L_\D(w^*) + (\alpha^2 - 1) L_\D(w^*) = \alpha^2 L_\D(w^*)
\]
and the proof is concluded.
\end{proof}

\end{document}